\documentclass[11pt,letterpaper]{article}
\pdfoutput=1 
\usepackage[height=8in]{geometry}
\usepackage{amsmath, amssymb, amsthm, comment}
\usepackage{graphicx, color}
\usepackage{array, multirow}
\usepackage[font=small,labelfont=bf]{caption} 
\usepackage{lipsum}
\usepackage{amsfonts}
\usepackage{tabularx}
\usepackage{booktabs}
\usepackage{graphicx}
\usepackage{epstopdf}
\usepackage{algorithmic}
\usepackage{mathtools}
\usepackage{tikz}
\ifpdf
  \DeclareGraphicsExtensions{.eps,.pdf,.png,.jpg}
\else
  \DeclareGraphicsExtensions{.eps}
\fi

\AtBeginDocument{%
	\setlength{\oddsidemargin}{\dimexpr(\paperwidth-\textwidth)/2-1in}%
	\setlength{\evensidemargin}{\oddsidemargin}%
	\setlength{\topmargin}{%
		\dimexpr(\paperheight-\textheight)/2-\headheight-\headsep-1in}%
}

\usepackage[shortlabels]{enumitem}

\usepackage[sort]{natbib}

\setcitestyle{numbers,square,comma}
\usepackage{hyperref}
\hypersetup{colorlinks=true,
			urlcolor=blue,
			linkcolor=blue,
			citecolor=blue,
			bookmarksdepth=paragraph}
\usepackage[nameinlink]{cleveref}
\crefname{equation}{}{}
\crefname{section}{section}{sections}
\crefname{figure}{figure}{figures}
\crefname{table}{table}{tables}
\crefname{example}{example}{examples}
\crefname{proposition}{proposition}{propositions}
\Crefname{section}{Section}{Sections}
\Crefname{figure}{Figure}{Figures}
\Crefname{table}{Table}{Tables}
\Crefname{definition}{Definition}{Definitions}
\Crefname{theorem}{Theorem}{Theorems}
\Crefname{remark}{Remark}{Remarks}
\Crefname{example}{Example}{Examples}
\Crefname{proposition}{Proposition}{Propositions}
\numberwithin{equation}{section}

\newtheorem{theorem}{Theorem}[section]
\newtheorem{lemma}{Lemma}[section]
\newtheorem{corollary}{Corollary}[section]
\newtheorem{proposition}{Proposition}[section]
\theoremstyle{definition}
\newtheorem{definition}{Definition}[section]
\newtheorem{remark}{Remark}[section]
\newtheorem{assumption}{Assumption}[section]

\title{Transferability of Graph Neural Networks using Graphon and Sampling Theories}

\author{A. Martina Neuman$^1$, Jason J.\ Bramburger$^2$}
\date{$^1$Faculty of Mathematics, University of Vienna, Vienna, Austria\\
$^2$Department of Mathematics and Statistics, Concordia University, Montr\'eal, QC, Canada}

\begin{document}

\maketitle

\begin{abstract}
Graph neural networks (GNNs) have become powerful tools for processing graph-based information in various domains. A desirable property of GNNs is transferability, where a trained network can swap in information from a different graph without retraining and retain its accuracy. A recent method of capturing transferability of GNNs is through the use of graphons, which are symmetric, measurable functions representing the limit of large dense graphs. In this work, we contribute to the application of graphons to GNNs by presenting an explicit two-layer graphon neural network (WNN) architecture. We prove its ability to approximate bandlimited graphon signals within a specified error tolerance using a minimal number of network weights. We then leverage this result, to establish the transferability of an explicit two-layer GNN over all sufficiently large graphs in a convergent sequence. 
Our work addresses transferability between both deterministic weighted graphs and simple random graphs and overcomes issues related to the curse of dimensionality that arise in other GNN results. The proposed WNN and GNN architectures offer practical solutions for handling graph data of varying sizes while maintaining performance guarantees without extensive retraining.
\end{abstract}

\paragraph{Key words:} graphon, regularized sampling, graph neural network, transferability, random graph, curse of dimensionality


\section{Introduction}

Graph neural networks (GNNs), falling under geometric deep learning models \citep{bronstein2021geometric}, are powerful tools for processing graph-structured data \citep{wu2022graph,scarselli2008graph,micheli2009neural}. They leverage graph topology to facilitate information exchange among neighboring vertices, finding applications in diverse domains such as natural language processing \citep{wu2023graph}, chemistry \citep{jiang2021could,fung2021benchmarking,gilmer2017neural}, citation networks \citep{bhagavatula2018content,cummings2020structured,hamilton2017inductive}, and recommender systems \citep{ruiz2020graphon,huang2021mixgcf,gao2022graph,wu2022graph,ricci2021recommender,resnick1997recommender,gao2023survey}. GNNs have been shown to exhibit superior predictive power compared to traditional neural networks \citep{ala2020improving,ma2021deep} as well as achieve efficient generalization of bandlimited functions with fewer network weights than the best-known results for deep neural networks \citep{neuman2022superiority}. For a comprehensive introduction to the theory and applications of GNNs, see the review by \citep{zhou2020graph}.

One of the most desirable properties of GNNs is transferability. This allows a GNN to be used across different graphs without necessitating re-training while maintaining performance guarantees. An instance of transferability would be ensuring robustness when there are alterations in the underlying edge or weight structure.
Another example of transferability would be the adaptability between graphs of varying sizes, as seen with recommender systems where vertices represent users, and edges denote similarities.
When users frequently join or leave the platform, yet the graph structure remains essentially consistent among those continuing on, a GNN-based system capable of curating recommendations for similar individuals across networks with fluctuating sizes is paramount.

To achieve transferability of GNNs across networks of varying sizes, \citep{ruiz2020graphon} introduced graphon neural networks (WNNs), in which graphons, represented by symmetric, measurable functions $W:[0,1]^2 \to [0,1]$, served as a key component. 
These limit objects capture the essence of large, dense graphs as the number of vertices approaches infinity \citep{lovasz2006limits,Janson,borgs2017sparse,glasscock2015graphon}.
Members of a sequence of graphs converging to a graphon $W$ can be seen as sharing structural characteristics, much along the line of recommender systems. 
This underscores the increasing use of graphons to capture transferability and other key properties of GNNs. 
As an example, graphons were used in \citep{ruiz2020graphon} to prove GNN transferability across large diverse graphs, with similar transferability results achieved in \citep{levie2021transferability}. Along the same vein, the transferability of spectral graph neural networks was demonstrated in \citep{maskey2023transferability}.
Further research explores GNN stability under graphon and graph perturbations \citep{ruiz2021graphon, keriven2020convergence}. Experimental studies such as \citep{ruiz2021transferability} evaluate the theoretical performance guarantees of GNNs trained on moderate-sized graphs when transferred to larger ones.
Graphons are also employed to estimate parameters for GNN training \citep{hu2021training}, and for pooling and sampling in GNNs, with the former application demonstrably enhancing resilience against overfitting \citep{parada2021graphon}. 
In a primarily theoretical context, \citep{cervino2023learning} proposes a method to learn WNNs by training GNNs on expanding networks. 
Finally, there is an emerging literature on signal processing leveraging graphons \citep{ruiz2021graphon,morency2021graphon,ruiz2021graphonsignal,ruiz2021graphonprocessing} to describe structurally similar graph signals over various networks.  

Our work contributes to the emerging application of graphons to data science by presenting a two-layer WNN architecture that can approximate a bandlimited signal within an $\varepsilon > 0$ error tolerance using $\mathcal{O}(\varepsilon^{-10/9})$ network weights.
This result is then leveraged to prove transferability of an explicit two-layer GNN among all sufficiently large graphs converging to a graphon, accounting for both deterministic weighted graphs and simple random graphs, of which the latter is often overlooked in the WNN literature. 
Importantly, the use of graphons enables us to circumvent any issues related to the curse of dimensionality, keeping, in particular, the number of layers bounded independently of $\varepsilon$. 
The reader should compare with similar results on approximating bandlimited functions on high-dimensional domains \citep{montanelli2021deep,chen2019note}, in which the numbers of weights and layers become necessarily unbounded as $\varepsilon \to 0^+$.

Our focus on bandlimited functions allows for a sampling result that underpins our WNN architecture and aligns with the growing literature on neural networks approximating bandlimited functions \citep{neuman2022superiority,chen2019note,montanelli2021deep,opschoor2022exponential,wang2018exponential,dziedzic2019band,wang2022convolutional}. We consider a bandlimited function as a linear combination of finitely many Fourier modes, unlike recent works that replace Fourier modes with eigenfunctions of the underlying graphon or manifold \citep{ruiz2021graphonprocessing,wang2022convolutional}. 
Since one rarely has access to the graphon and graphon signals beyond finite samples, we view working with Fourier bandlimited functions as more practical. Nonetheless, for ring graphons, which depend only on the distance $|x - y|$, the eigenfunctions are exactly the Fourier modes, making the two bandlimited definitions coincide \citep{bramburger2023pattern}.

We summarize our results in this paper as follows:
\begin{enumerate}
    \item We prove a sampling theorem that gives an explicit reconstruction of bandlimited functions on $[0,1]$, with a subgeometric convergence rate, providing a foundation for our WNN and GNN architectures.
    \item Guided by our sampling theorem, we present a two-layer WNN architecture that can reproduce bandlimited graphon signals with an $L^2$-error $\varepsilon > 0$ using only $2N = \mathcal{O}(\varepsilon^{-10/9})$ evenly spaced samples from $[0,1]$.
    \item Through a discretization of the interval $[0,1]$, the WNN architecture leads to a GNN architecture for which graph adjacency matrices are simply swapped into the filter computational unit. We establish performance guarantees on the transferability of these GNNs for all sufficiently large simple deterministic weighted graphs and random graphs belonging to the same graphon family.   
\end{enumerate}

We outline in Section~\ref{sec:Preliminaries} all important concepts, definitions, and auxiliary results to be used throughout. Precisely, Subsection~\ref{sec:GNNIntro} reviews GNNs, Subsection~\ref{sec:graphons} introduces graphons, and Subsection~\ref{sec:WNN} defines WNNs as a generalization of GNNs. Subsection~\ref{sec:Fsamptheory} discusses bandlimited functions, and Subsection~\ref{sec:WNNframework} presents our specific WNN and GNN architectures. Our main results are given in Section~\ref{sec:Results}, organized into subsections according to the summarized contributions above, followed by a discussion in Subsection~\ref{sec:Ramifications} on their ramifications. The proof of our sampling theorem is left to Section~\ref{sec:Tsamplingthm}, the proof of WNN generalization to Section~\ref{sec:WNNthm}, and the GNN proofs to Sections~\ref{sec:GNNdet} and~\ref{sec:GNNran}. Section~\ref{sec:Discussion} concludes with a discussion of our results and some important avenues for future investigation.

\section{Preliminaries}\label{sec:Preliminaries}

We introduce the symbols, notations, and conventions that will be used consistently throughout this paper. 
For $N\in\mathbb{N}$, we define an integer interval $B_N$ to be
\begin{equation*}
    B_N := \{-N, -N+1, \cdots, 0, \cdots, N-2, N-1\}. 
\end{equation*}
We denote the cardinality of a set $X$ by $\texttt{\#} X$. Throughout we will employ a conventional abuse of the notation $|\cdot|$ so that, when $I$ is an interval, $|I|$ means the length of $I$, but when $x$ is a Euclidean element, $|x|$ means its Euclidean norm.  

Let $p=1,2$. When $X$ is a Lebesgue measurable set, the space $L^p(X;\mathbb{C}^m)$ comprises all functions $f:X \to \mathbb{C}^m$ for which the norm
\begin{equation*}
    \|f\|_{L^p(X;\mathbb{C}^m)} := \bigg(\int_X |f(x)|^p\,\mathrm{d}x\bigg)^{1/p}
\end{equation*}
is finite. 
When $X$ is a discrete set equipped with the counting measure, $f\in L^p(X;\mathbb{C}^m)$ if it has a finite norm
\begin{equation*}
    \|f\|_{L^p(X;\mathbb{C}^m)} := \bigg(\sum_{x \in X} |f(x)|^p \bigg)^{1/p}.
\end{equation*}
We preserve the symbols $\ell^1$ and $\ell^2$ for when $X=\mathbb{Z}$.
When $m = 1$, we simply write $L^p(X;\mathbb{C})$ as $L^p(X)$ and $\ell^p(\mathbb{Z};\mathbb{C})$ as $\ell^p(\mathbb{Z})$. 

Another abuse of notation arises with the congruence symbol $\cong$. When $X, Y$ are sets, $X \cong Y$ means they are isomorphic as sets, for example, the torus $\mathbb{T} \cong [0,1)$ or $\mathbb{T} \cong [-1/2,1/2)$. When $X,Y$ are Lebesgue measurable sets, $X \cong Y$ means their set difference $X \Delta Y$ has Lebesgue measure zero, for example, $[0,1] \cong [0,1)$. Lastly, when $X, Y$ are groups, $X \cong Y$ signifies that they are isomorphic as groups; however, this usage only appears in Appendix~\ref{sec:Tsamp}.

\subsection{Graph Neural Networks}\label{sec:GNNIntro}

Let $G = (V,E,w)$ denote a graph with $V$ vertices, $E\subset V\times V$ graph edges, and an edge weight function $w: V\times V\to [0,1]$.
Supposing $\texttt{\#}V =n$, for $n\in\mathbb{N}$, we write $V = \{v_1,v_2,\dots,v_n\}$. 
We restrict ourselves to the case of $G$ being a simple, undirected graph, meaning that $w$ is symmetric, i.e. $w(v_k,v_l) = w(v_l,v_k)$, and contains no self-loops, i.e. $w(v_k,v_k)=0$. Associated with $G=(V,E,w)$ is the graph adjacency matrix ${\bf A}\in\mathbb{R}^{n\times n}$ satisfying $[{\bf A}]_{kl}=w(v_k,v_l)$. 

A {\em graph signal} is a function $f\in L^2(V;\mathbb{C}^m)$, for some $m \in \mathbb{N}$. 
A GNN $\Psi_G$ on $G$ can also be viewed as a function in $L^2(V; \mathbb{C}^m)$.
Often in many applications of GNNs, vertices $v_k$ are assigned coordinate vectors $X_k\in\mathbb{R}^d$ that depict their locations in space. 
Then $\Psi_G$ takes as input the collection of all these feature vectors $X_k$, known as the (input) feature matrix:
\begin{equation}\label{def:featuremat}
    {\bf X} := \begin{bmatrix} 
        X_1 & X_2 & \cdots & X_n
    \end{bmatrix}\in \mathbb{R}^{d\times n}.
\end{equation} 
The output of $\Psi_G$ at each vertex $v_k$ is a feature vector $Y_k\in\mathbb{C}^m$. The collection of these is the (output) feature matrix
\begin{equation*}
    {\bf Y} := \begin{bmatrix} 
        Y_1 & Y_2 & \cdots & Y_n
    \end{bmatrix}  \in \mathbb{C}^{m\times n}.
\end{equation*}
The input and output dimensions, respectively, $d, m$, are problem-dependent. 

In terms of structure, a GNN is characteristically defined by finite compositions of two distinct computational units:
\begin{itemize}
    \item {\bf Multilayer Perceptron (MLP)}: Let ${\bf W} \in\mathbb{C}^{m\times d}$ be a {\it weight matrix}, ${\bf b}\in\mathbb{C}^{m}$ a {\it bias vector}, and $\rho$ a nonlinear {\it activation function}. Then an MLP sends ${\bf Z}\in\mathbb{C}^{d\times n}$ to $\rho({\bf W}{\bf Z} + {\bf b})$, 
    where $\rho$ is applied componentwise. 
    \item {\bf Filter}: A graph filter $\mathfrak{F}$ distinguishes a GNN from a standard feedforward NN architecture. Generally, a graph filtering process is any process that takes in nodal features (e.g. {\bf X} in \eqref{def:featuremat}) {\it and} the graphical structure, to outputs a new set of nodal features. There are typically two types of graph filters: spatial-based and spectral-based. The spatial filters explicitly leverage the graph connections to perform a feature refining process, whereas the spectral ones utilize spectral graph theory to design filtering in the spectral domain 
    \citep{ma2021deep}.
\end{itemize}

Given these points, the total architecture of a GNN can be expanded as
\begin{equation}\label{gnnex}
    {\bf Y}\equiv \rho_{L}({\bf W}^{(L)} \psi_{L-1}\circ\psi_{L-2} \circ\cdots\circ \psi_1({\bf X})+{\bf b}^{(L)}).
\end{equation}
Here, $L\in\mathbb{N}$ is the number of layers, and each $\psi_{j}$ in \eqref{gnnex} is a (intermediate) {\it network layer}, where
\begin{equation}\label{GNNlayer}
    \begin{split}
        &{\bf H}_{j} = \psi_{j}({\bf H}_{j-1}) = \mathfrak{F} (\rho({\bf W}^{(j)}{\bf H}_{j-1} + {\bf b}^{(j)})), \quad  j=1,\cdots, L-1,\\
        &{\bf H}_0 \equiv {\bf X}.
    \end{split}
\end{equation}
In \eqref{gnnex}, $\mathfrak{F}$ is not performed at the last $L$th layer, and the last activation $\rho_{L}$ is optional. 
Furthermore, $\mathfrak{F}$ should be enabled at a minimum of one - but not necessarily all - intermediate layers. The GNN {\it network parameter} is the set of all entries of all the weight matrices, all the bias vectors, and the number of layers.

In this work, we use the layout \eqref{GNNlayer}, noting that alternatives exist where the computational units are arranged differently. 
We will also focus on (linear) spatial-based graph filters. Such filtering 
is implemented by a {\it graph filter kernel} (sometimes referred to as a {\em graph shift operator} \citep{ruiz2020graphon}) ${\bf K} \in \mathbb{C}^{n \times n}$ which exploits the graph topology of $G$. 
Specifically, if ${\bf Z}\in\mathbb{C}^{d\times n}$, then 
\begin{equation}\label{GSO}
    \mathfrak{F}: \quad {\bf Z} \mapsto {\bf K}{\bf Z}^{\top}.
\end{equation} 
The choice of ${\bf K}$ can be determined by the user and is often problem-dependent. We will explore graph filter kernels of the form ${\bf K} = {\bf G} \circ {\bf A}$, where $\circ$ denotes the Hadamard product, and ${\bf G}\in\mathbb{C}^{n\times n}$ localizes information propagation by inducing a near-sparsity pattern in ${\bf K}$. This sparsification effect is particularly helpful in dense graphs, ensuring that information ${\bf Z}$ in \eqref{GSO} does not spread globally across $G$ through matrix multiplication by ${\bf K}^{\top}$.

\subsection{Graphons} \label{sec:graphons}

A {\em graphon} is a symmetric, Lebesgue measurable function $W:[0,1]^2 \to [0,1]$, serving as a continuum generalization of a graph adjacency matrix. 
This interpretation comes from envisioning $[0,1]$ as a graph of an uncountable number of vertices, where $W(x,y)\in [0,1]$ represents the edge weight 
between two arbitrary vertices $x,y \in [0,1]$ - as with weighted graphs, one interprets $W(x,y) = 0$ as no edge being present. 

\paragraph{Graphons as limit objects.} Graphons arise as limits of sequences of graphs on increasing numbers of vertices. To interpret graph convergence meaningfully, we introduce ``homomorphism density" from graph $G$ to graph $H$ as follows:
\begin{equation*}
    t(G,H) := \frac{\texttt{\#} \mathrm{Hom}(G,H)}{ (\texttt{\#} V_H)^{\texttt{\#} V_G}},
\end{equation*}
where $V_H$, $V_G$ denote the number of vertices in graphs $G$, $H$, respectively, and $\texttt{\#} \mathrm{Hom}(G,H)$ denotes the number of {\it graph homomorphisms} from $G$ to $H$. Thus, $t(G,H)$ represents the probability that a mapping from $G$ to $H$ is a graph homomorphism. A sequence of graphs $\{G_n\}_{n = 1}^\infty$ is said to converge if for all finite simple graphs $F$, the limit of $\{t(F,G_n)\}_{n = 1}^\infty$ exists. 
The homomorphism density of a graph $F$ into a graphon $W$ can be defined in an analogous manner \citep[Chapter~7.2]{lovasz2012large}, resulting in a ratio $t(F,W)$. Then, as shown in \citep[Chapter~11]{lovasz2012large}, for every convergent sequence of graphs $\{G_n\}_{n = 1}^\infty$ there exists a graphon $W$ so that 
\begin{equation}\label{graphlim} 
    \lim_{n \to \infty} t(F,G_n) = t(F,W),
\end{equation}
for every finite simple graph $F$. Through \eqref{graphlim}, we may think of $\{G_n\}_{n=1}^\infty$ as belonging to the same graphon family. 

\paragraph{Generating graphs from graphons.}  A major application of graphons is to generate graphs on finitely many vertices. This is done by discretizing the unit interval $[0,1]$ into $n\in\mathbb{N}$ points as
\begin{equation} \label{xj}
    x_k := \frac{k - 1}{n}, \quad k = 1,\cdots,n,
\end{equation}
which represent graph vertices. We write $\mathcal{X}_n:=\{x_1, \cdots, x_n\}$ and $I_k := [x_k, x_{k+1}) = [(k-1)/n, k/n)$, for $k=1,\cdots, n$. 

Using $W$, we construct a deterministic weighted graph $G^{\mathrm{det}}_n$ on $v_1, \cdots, v_n$ (identified with $x_1, \cdots, x_n$) by defining its adjacency matrix ${\bf A}^{\mathrm{det}}_n$ as
\begin{equation*}
    [{\bf A}^{\mathrm{det}}_n]_{k,l} := \begin{cases}
        W(x_k,x_l) & k \neq l\\
        0 & k = l
    \end{cases}.
\end{equation*}
Similarly, $W$ can be used to generate a simple {\it random graph} $G^{\mathrm{ran}}_n$ on $v_1, \cdots, v_n$. In this case, the associated (random) adjacency matrix ${\bf A}^{\mathrm{ran}}_n$ is given by
\begin{equation*}
    [{\bf A}^{\mathrm{ran}}_n]_{k,l} := \begin{cases}
        \xi_{k,l} & k \neq l\\
        0 & k = l
    \end{cases}, 
\end{equation*}
where, for $k>l$, $\xi_{k,l} = \xi_{l,k}$ are independent Bernoulli random variables such that
\begin{equation*}
    \mathbb{P}(\xi_{k,l} = 1) = 1 - \mathbb{P}(\xi_{k,l} = 0) = W(x_k,x_l).
\end{equation*}

The graphs in both sequences $\{G^{\mathrm{det}}_n\}_{n = 1}^\infty$ and $\{G^{\mathrm{ran}}_n\}_{n = 1}^\infty$ belong to the graphon family of $W$. Indeed, for all finite simple graphs $F$, one has \citep[Chapter~11]{lovasz2012large}
\begin{equation*}
    \lim_{n \to \infty} t(F,G^{\mathrm{det}}_n) = t(F,W), 
\end{equation*}
and with probability 1
\begin{equation*}
    \lim_{n \to \infty} t(F,G^{\mathrm{ran}}_n) = t(F,W).
\end{equation*}

\paragraph{Embedding graphs into the graphon space.}
Graphs can be lifted into the space of graphons. Let $G = (V,E,w)$ be a finite simple graph with $V=\{v_1, \cdots, v_n\}$. We define the {\it step graphon} $W_n:[0,1]^2 \to [0,1]$ by
\begin{equation}\label{def:tildeW_n}
    \overline{W}_n(x,y) := \sum_{k,l = 1}^n w(v_k,v_l)\chi_{I_k\times I_l}(x,y).
\end{equation}
In particular, when $G$ is a simple deterministic graph generated from a graphon $W$, \eqref{def:tildeW_n} becomes
\begin{equation*}
    \overline{W}_n(x,y) = \sum_{k,l = 1}^n w(v_k,v_l)\chi_{I_k\times I_l}(x,y) = \sum_{\substack{k,l=1\\ k\not= l}}^n W(x_k, x_l)\chi_{I_k\times I_l}(x,y).
\end{equation*}

\paragraph{Graph and graphon signals.} A {\it graphon signal} is a function in $L^2([0,1];\mathbb{C}^m)$, for some $m\in\mathbb{N}$. An $m$-dimensional graph signal $g$ on a graph $G$ of $n$ vertices has an associated graphon signal representative. By identifying the vertices $v_k$ with $x_k$ \eqref{xj}, we can view $g$ as defined on $\mathcal{X}_n$: $g(v_k) = g(x_k)$. We then embed $g$ into $L^2([0,1];\mathbb{C}^m)$ as the step function
\begin{equation*}
    \overline{g}(x) := \sum_{k = 1}^n g(x_k)\chi_{I_k}(x) = \sum_{k = 1}^n g(v_k)\chi_{I_k}(x), \quad\forall x\in [0,1].
\end{equation*}

\subsection{Graphon Neural Networks} \label{sec:WNN}

Drawing upon the analogy between graphs and graphons, we define a WNN as an extension of the GNN framework. Specifically, a WNN $\Psi$ is a function in $L^2([0,1];\mathbb{C}^m)$, constructed by iteratively composing two distinct computational units:
\begin{itemize}
    \item {\bf MLP}: 
    Let ${\bf W} \in\mathbb{C}^{m\times d}$ be a weight matrix, $b\in L^2([0,1];\mathbb{C}^m)$ a {\it bias function}, and $\rho$ a nonlinear activation function. Then an MLP acts on $g\in L^2([0,1];\mathbb{C}^d)$ as
    \begin{equation*} 
        g(x) \mapsto \rho({\bf W}g(x) + b), \quad \forall x\in [0,1],
    \end{equation*}
    where $\rho$ is applied componentwise. 
    \item {\bf Graphon filter}: A graphon filter generalizes a graph filter. Let $\mathcal{K}\in L^{\infty}([0,1]^2)$. Then the graphon filter $T_\mathcal{K}$ acting on $g = (g_1,\cdots,g_m)\in L^2([0,1];\mathbb{C}^{m})$ as
    \begin{equation*}
        \begin{split}
            T_\mathcal{K}g(x) &= (T_\mathcal{K}g_1(x), \cdots, T_\mathcal{K}g_m(x))\\
            &:= \bigg(\int_0^1 \mathcal{K}(x,y)g_1(y)\,\mathrm{d}y, \cdots, \int_0^1 \mathcal{K}(x,y)g_m(y)\,\mathrm{d}y\bigg), \quad\forall x\in [0,1].
        \end{split}
    \end{equation*} 
\end{itemize}

The choice of $\mathcal{K}$ can be either user or problem determined, and taking $\mathcal{K}$ to be a graphon $W$ gives the operator $T_W$ that is well-studied in the literature \citep{Janson,lovasz2012large,ruiz2021graphonsignal}. Adopting the configuration \eqref{GNNlayer}, an overall WNN structure can be similarly symbolized as
\begin{equation*}
    \Psi(g_0)(x) \equiv \rho_{L}({\bf W}^{(L)} \psi_{L-1}\circ\psi_{L-2} \circ\cdots\circ \psi_1(g_0)(x) +b^{(L)}), \quad\forall x\in [0,1].
\end{equation*}
Here, the graphon signal $g_0\in L^2([0,1];\mathbb{C}^{d})$, and for each intermediate layer, 
\begin{equation*}
    g_{j}(x) = \psi_{j}(g_{j-1})(x) = T_{\mathcal{K}} (\rho({\bf W}^{(j)}g_{j-1}(x) + b^{(j)})), \quad  j=1,\cdots, L-1.
\end{equation*}

\subsection{Bandlimited graphon signals and sampling} \label{sec:Fsamptheory}

Let $f=(f_1,\cdots, f_m): [0,1]\to\mathbb{C}^m$ be a graphon signal. We assume $f\in L^2([0,1];\mathbb{C}^{m})$, and therefore $f$ can be expressed as a Fourier series,
\begin{equation*}
    f(x) = \sum_{k\in\mathbb{Z}} \hat{f}(k)e^{i2\pi kx},
\end{equation*}
where $\hat{f}(k)=(\hat{f}_1(k),\cdots,\hat{f}_{m}(k))$, and the equal sign equates two members of $L^2([0,1];\mathbb{C}^m)$.  We are interested in profiles $f$ whose Fourier coefficients 
\begin{equation}\label{Fcoef}
    \hat{f}(k)=(\hat{f}_1(k),\cdots,\hat{f}_{m}(k)):=\int_0^1 f(x)e^{-i2\pi kx}\,\mathrm{d}x 
\end{equation} 
are zero for all but a finite number of modes $k\in\mathbb{Z}$. To avoid confusion with another notion of bandlimitedness circulated in the current graphon literature \citep{ruiz2020graphon}, we call such $f$ {\it Fourier bandlimited}.

\begin{definition} \label{def:Fblmtd}
Let $\mathfrak{m}\in\mathbb{N}$ and $f=(f_1,\cdots,f_{m})\in L^2([0,1];\mathbb{C}^{m})$. Then $f$ is said to be $\mathfrak{m}$-Fourier bandlimited, or $f\in\mathcal{B}_\mathfrak{m}$, if
\begin{equation*}
    \hat{f}(k) = 0, \quad \forall k\not\in B_\mathfrak{m} = \{-\mathfrak{m},\cdots,0,\cdots,\mathfrak{m}-1\}.
\end{equation*}
or equivalently, 
\begin{equation*}
    \hat{f}_j(k) = 0, \quad \forall k\not\in B_\mathfrak{m} = \{-\mathfrak{m},\cdots,0,\cdots,\mathfrak{m}-1\}
\end{equation*}
for every $j=1,\cdots,m$.
\end{definition}

It is well-known that a Fourier bandlimited $f$ is almost everywhere equal to a continuous function $f_{c}$ on $[0,1]$ where $f_{c}(0)=f_{c}(1)$ and
\begin{equation}\label{contf}
    f_{c}(x) = \sum_{k=-\mathfrak{m}}^{\mathfrak{m}-1} \hat{f}(k) e^{i2\pi kx}. 
\end{equation}
Therefore, we can and will identify $f$ with this continuous version on $[0,1]$ and on $\mathbb{T}\cong [0,1)$. A function that is Fourier bandlimited is significant in the sense that all its information is stored in its spatial samplings, as encapsulated by the following proposition.

\begin{proposition}\label{prop:gpsampling} 
Let $f\in \mathcal{B}_\mathfrak{m}$, and let $N\geq\mathfrak{m}$. Then, 
\begin{equation}\label{gpsampcited}
    f(x) = \sum_{j=0}^{2N-1} f\Big(\frac{j}{2N}\Big)s_N(x-j/2N), \quad \forall x\in [0,1],
\end{equation}
where 
\begin{equation}\label{def:sampf}
    s_N(x):=\frac{1}{2N}\sum_{k=-N}^{N-1} e^{i2\pi kx}, \quad \forall x\in [0,1].
\end{equation}
\end{proposition}

Proposition~\ref{prop:gpsampling} is a straightforward consequence of the Kluvan\'ek's sampling theorem~\citep{kluvanek1965sampling}, with the proof being left to Appendix~\ref{sec:Tsamp}. The function $s_N$ 
is referred to as a {\it sampling function}. The reasoning behind the choice of $s_N$, as well as a derivation of \eqref{gpsampcited}, is fully detailed in Appendix~\ref{sec:Tsamp} - see in particular \eqref{sfunc} and \eqref{sampseries}. One notable observation from the requirement $N\geq\mathfrak{m}$ is that, when $N=\mathfrak{m}$, the sampled values $\{f(j/2N)\}_{j=0}^{2N-1}$ are said to be sampled at the Nyquist rate \citep{benedetto2001modern}.


We conclude this subsection with the following useful lemma that serves as a fundamental connection between sampling theory and Fourier theory on the torus group $\mathbb{T}$. The second conclusion of the lemma, in particular, provides a practical means for accessing the total graphon signal energy through its spatial samples. This crucial point will be a recurring element in our analysis. The proof is again left to Appendix~\ref{sec:Tsamp}. 

\begin{lemma} \label{lem:tech} 
Let $f\in \mathcal{B}_\mathfrak{m}$, and let $N\geq\mathfrak{m}$. Then, for $\hat{f}$ to be as in \eqref{Fcoef}, we have
\begin{equation}\label{claim_finv}
    \hat{f}(k) = \begin{cases} 
                    \frac{1}{2N}\sum_{j=0}^{2N-1} f\Big(\frac{j}{2N}\Big)e^{-i2\pi kj/2N} &k\in B_\mathfrak{m}\\
                    0 &\mathrm{ otherwise}
                 \end{cases},
\end{equation}
and
\begin{align}\label{GquadforT}
    \nonumber \|f\|_{L^2([0,1];\mathbb{C}^m)}^2 = \|f\|_{L^2(\mathbb{T};\mathbb{C}^m)}^2 &= \|\hat{f}\|_{\ell^2(\mathbb{Z};\mathbb{C}^m)}^2 \\
    &= \sum_{k=-\mathfrak{m}}^{\mathfrak{m}-1} |\hat{f}(k)|^2 = \sum_{k=-N}^{N-1} |\hat{f}(k)|^2 = \sum_{j=0}^{2N-1} |f\Big(\frac{j}{2N}\Big)|^2.
\end{align}   
\end{lemma}

\subsection{WNN and GNN architectures enabled by the sampling theory} \label{sec:WNNframework}

We describe the specific network architectures used to achieve our main results, setting the nonlinear activation $\rho = {\rm ReLU}$. The resulting networks are called ReLU WNN and ReLU GNN.
Moreover, throughout the rest of this paper, we write $\rho_j = \rho_{\frac{j}{2N}} := {\rm ReLU}(\cdot - \frac{j}{2N})$, for $j=1,\cdots, 2N-1$.

\paragraph{WNN architecture.} Let $f: [0,1]\to\mathbb{C}^m$ and suppose that for some $N\geq 1$ we have access to the vector
\begin{equation*}
    f_\mathrm{samp} = \begin{bmatrix} f(0) & f\big(\frac{1}{2N}\big) & \cdots & f\big(\frac{2N-1}{2N}\big)\end{bmatrix} \in \mathbb{C}^{m \times 2N}.
\end{equation*}
The objective is to design a WNN that can predict the value of $f$ at $x\in[0,1]$ using only the knowledge of $f_\mathrm{samp}$. Our WNN architecture, using ReLU activations and a family of graphon filters $\{T_{\mathcal{K}_j}\}_{j=0}^{2N-1}$:
\begin{equation}\label{TcalK}
    T_{\mathcal{K}_j} g(x) = \int_0^1 \mathcal{K}_j (x,y)g(y) \,\mathrm{d}y. 
\end{equation}
It comprises two hidden layers, a filter layer and an NN layer, and is built according to the following steps.

\begin{enumerate}
    \item {\bf Input:} 
    $g_0(x)= x$ for all $x\in [0,1]$.
    \item {\bf First hidden layer:} 
    The output function $g_1$ of the first layer is given by an application of affine transformation followed by a componentwise activation and a graphon filtering:
    \begin{align*}
        g_0(x) = x \mapsto
        \begin{bmatrix} x \\ x-\frac{1}{2N} \\ \vdots \\ x-\frac{2N-1}{2N} \end{bmatrix} 
        \overset{\rho}{\mapsto}  
        \begin{bmatrix} \rho(x) \\ \rho (x-\frac{1}{2N}) \\ \vdots \\ \rho(x-\frac{2N-1}{2N}) \end{bmatrix} 
        \overset{\{T_{\mathcal{K}_j}\}_j}{\mapsto} 
        \begin{bmatrix} T_{\mathcal{K}_0} \rho(x) \\ T_{\mathcal{K}_1} \rho_1 (x) \\ \vdots \\ T_{\mathcal{K}_{2N-1}} \rho_{2N-1}(x) \end{bmatrix} = g_1(x).
    \end{align*}
    \item{\bf Second hidden layer:} 
    The output $g_2$ of the second layer is given by
    \begin{equation} \label{WNNsamplingform}
         g_1(x) \mapsto 
         \sum_{j=0}^{2N-1} f\Big(\frac{j}{2N}\Big)T_{\mathcal{K}_j}\rho_j (x) \equiv g_2(x). 
    \end{equation}
    \item{\bf Output:} 
    $g_2(x)$ for all $x\in [0,1]$.
\end{enumerate}

The formula for the WNN output \eqref{WNNsamplingform} resembles 
\eqref{gpsampcited} on $[0,1]$. As we progress further in this paper, it will become apparent that \eqref{WNNsamplingform} closely approximates \eqref{gpsampcited} (and hence, $f$),
demonstrating the network's effective utilization of the sampling principle, through an application of the graphon kernels 
\begin{equation*} 
    \mathcal{K}_j(x,y) = \chi_{\{x-r\leq j/2N\}}(x)\mathcal{G}^*(x,y)W(x,y).
\end{equation*}
Here, $\mathcal{G}^*$ is a ``Gaussian-like" function whose specifics will be presented along with our main results. 
We 
assume that $W$ possesses the following local regularity. 

\begin{assumption} \label{assump:regular} There exist $\kappa, \eta\in (0,1)$, $K>0$, such that for every $x\in [0,1]$ and almost every pair $(x,y)$ and $(x,z)$ in the diagonal region
\begin{equation}\label{def:diagregion}
    \mathcal{D}_\kappa := \{(x,y)\in [0,1]^2: |x-y|\leq \kappa\},
\end{equation}
we have
\begin{equation}\label{regularcond}
    W(x,y) \geq \eta \quad \text{ and }\quad |W(x,y)-W(x,z)|\leq K|y-z|.
\end{equation}
\end{assumption}

We emphasize that we do not believe this assumption to be overly restrictive in practice. The second condition in \eqref{regularcond} is a (one-dimensional) {\it local Lipschitz} condition, while the first, the {\it nonvanishing} condition $W(x,y) \geq \eta$ is technical. 
If 
$W(x,y)=0$ for almost every $y\in [-\kappa + x, x+\kappa]$, then any WNN using $W$ will fail to retain 
the value $\tilde{f}(x)$ for all graphon signals $\tilde{f}$ that are essentially supported in a neighborhood of $x$: 
\begin{equation*}
    \bigg|\int_0^1 W(x,y) \tilde{f}(y)\,\mathrm{d}y \bigg| = \text{ small }\not= |\tilde{f}(x)|, 
\end{equation*}
indicating that the filtering dissipates information regarding $\tilde{f}$ at $x$.

\paragraph{GNN architecture.} 
Running parallel to the WNN architecture presented is our GNN architecture. 
Let $G_n$ be a graph on $n$ vertices generated from a graphon $W$, 
where 
each vertex $v_k$ 
is identified with $x_k = \frac{k-1}{n}$, for $k=1,\cdots,n$. Let $f_n$ be an associated graph signal, 
which can be viewed as a function on $\mathcal{X}_n = \{x_1, \cdots, x_n\}$ via $f_n(x_k)=f_n(v_k)$. 
Suppose there exists a graphon signal $f$ such that $f_n(x_k) = f(x_k)$ and that, for a fixed $N \geq 1$, we have sampled 
\begin{equation*} 
    f_\mathrm{samp} = \begin{bmatrix} f(0) & f\big(\frac{1}{2N}\big) & \cdots & f\big(\frac{2N-1}{2N}\big)\end{bmatrix} \in \mathbb{C}^{m \times 2N}.
\end{equation*}
We design a GNN that can predict the value of $f_n$ at $x_k$ knowing only $f_\mathrm{samp}$. 
It uses ReLU activations and a family of graph filters $\{\mathfrak{F}_j\}_{j=0}^{2N-1}$:
\begin{equation}\label{frakF}
    \mathfrak{F}_j(g)(x_k) = \frac{1}{n}\sum_{l\not=k, l=1}^{n}\mathcal{K}_{n,j}(x_k,x_l)g(x_l),
\end{equation}
where $\mathcal{K}_{n,j}$ is an appropriate graph filter kernel. Notably, \eqref{frakF} is a discretization of \eqref{TcalK}. 
The architecture comprises two hidden layers, a filter layer and an NN layer, described as follows.

\begin{enumerate}
    \item {\bf Input:} 
    $g_0(x_k)= x_k$ for each $k = 1,\cdots, n$.
    \item {\bf First hidden layer:} 
    The output function $g_1$ of the first layer is given by an application of affine transformation followed by a componentwise activation and a graphon filtering:
    \begin{align*}
        g_0(x_k) = x_k \mapsto
        \begin{bmatrix} x_k \\ x_k-\frac{1}{2N} \\ \vdots \\ x_k-\frac{2N-1}{2N} \end{bmatrix} 
        \overset{\rho}{\mapsto}  
        \begin{bmatrix} \rho(x_k) \\ \rho (x_k-\frac{1}{2N}) \\ \vdots \\ \rho(x_k-\frac{2N-1}{2N}) \end{bmatrix} 
        \overset{\{\mathfrak{F}_j\}_j}{\mapsto} 
        \begin{bmatrix} \mathfrak{F}_0 (\rho)(x_k) \\ \mathfrak{F}_1 (\rho_1)(x_k) \\ \vdots \\ \mathfrak{F}_{2N-1} (\rho_{2N-1})(x_k) \end{bmatrix} = g_1(x_k).
    \end{align*}
    \item{\bf Second hidden layer:} 
    The output $g_2$ is given by
    \begin{equation} \label{GNNoutput}
         g_1(x_k) \mapsto 
         \sum_{j=0}^{2N-1} f\Big(\frac{j}{2N}\Big)\mathfrak{F}_j(\rho_j)(x_k) \equiv g_2(x_k).
    \end{equation}
    \item{\bf Output:} 
    $g_2(x_k)$ for each $k = 1,\dots, n$.
\end{enumerate}

\begin{remark} \label{rem:inputfeat} {\it ($d$-dimensional input features versus one-dimensional input features)}
Our GNN construction exclusively operates on the input features $x_k \in [0,1]$ assigned to the vertices $v_k$ during the embedding of $G_n$ into the graphon space.
This approach, enabled by the presence of the graphon, simplifies inputs to single scalar values, contrasting with the conventional GNN framework of incorporating input feature vectors in $\mathbb{R}^d$.  
Nevertheless, these input features can still influence the structure of the limiting graphon $W$ in the following manner. Often, a graph is created from a data cloud using a geometric graph ($\varepsilon$-graph or $k$NN graph), where the neighborhood connections are established based on the proximity of the input features. 
This graph topology, crucial for graph filtering, is subsequently retained in $W$. 
\end{remark}

\begin{remark} \label{rem:sampling} {\it (embedded sampling principle)}
As demonstrated, we use the order \eqref{GNNlayer} for both networks to enable a seamless integration of sampling principle \eqref{gpsampcited} into their architectures \eqref{WNNsamplingform}, \eqref{GNNoutput}. This approach provides all necessary parameters for effective generalization, eliminating the need for learning. A similar demonstration can be found \citep{neuman2022superiority}.
\end{remark}

\begin{remark} \label{rem:bandlimited} {\it (bandlimitedness)} 
As indicated by \eqref{WNNsamplingform}, \eqref{GNNoutput}, the limiting graphon signal $f$ is required to be Fourier bandlimited. 
We can straightforwardly relax this requirement to include graphons whose energy is largely captured by a finite number of Fourier modes. 
In this case, the total error would consist of a generalization error of our GNN model and a bandlimited approximation error. 
\end{remark}

\section{Main Results}\label{sec:Results}

We present four theorems: a uniform sampling theorem on $[0,1]$ in Subsection~\ref{sec:SampResult}, a WNN generalization theorem in Subsection~\ref{sec:WNNResult}, and transferability theorems for GNNs in Subsection~\ref{sec:GNNthms}, applicable to both deterministic and random graphs from a graphon family. 
We conclude by discussing the ramifications of our results in Subsection~\ref{sec:Ramifications}.
All proofs are left to the sections that follow this one.

We adhere to the analysis practice where universal constants, denoted by $C$, may vary in value from one instance to another. 

\subsection{A uniform sampling theorem}\label{sec:SampResult}

Let $r \in (0,1/2)$ and define the 1-periodic function $\mathcal{G}_{r,\sigma}$ on $\mathbb{R}$ to be
\begin{equation}\label{def:Grsigma}
    \mathcal{G}_{r,\sigma}(x) := \begin{cases}
                            \mathcal{G}_{\sigma}(x) := c(\sigma)\sum_{n\in\mathbb{Z}} e^{-n^2\sigma^{-2}/2}e^{i2\pi n x}, & -r \leq x\leq r\\
                            0, &\mathrm{otherwise}
                          \end{cases}
\end{equation}
where $r$ is a truncation parameter, $\sigma$ is a variance parameter, and $c(\sigma)>0$ is a normalization constant such that $\mathcal{G}_{r,\sigma}(0)=\mathcal{G}_{\sigma}(0)=1$. Due to periodicity, $\mathcal{G}_{r,\sigma}$ can equivalently be viewed as a function on $\mathbb{T}\cong [0,1]$. Suppose $N>\mathfrak{m}$. For $f\in \mathcal{B}_\mathfrak{m}$, 
we consider a reconstruction scheme 
given by
\begin{equation}\label{reconstruction}
    \mathcal{R}_{r,\sigma}f(x) := \sum_{j=0}^{2N-1} f\Big(\frac{j}{2N}\Big)s_N(y-j/2N)\mathcal{G}_{r,\sigma}(x-j/2N),
\end{equation}
where $s_N$ is as in \eqref{def:sampf}. The following theorem is 
shows that, with the right choice of $r,\sigma$ and $N$, 
\eqref{reconstruction} produces a subgeometric convergence rate in $L^2$ norm. It will enable our results on WNNs and GNNs that follow. The proof is given in Section~\ref{sec:Tsamplingthm}.

\begin{theorem} \label{thm:TsamplingregwG}
Let $f\in \mathcal{B}_\mathfrak{m}$ such that $\|f\|_{L^2([0,1];\mathbb{C}^m)}=1$. 
Let $\mathcal{R}_{r,\sigma}f$ be as in \eqref{reconstruction}. Let $0<\alpha<\beta<1$, $N>\mathfrak{m}$, and set 
\begin{equation}\label{rsigmapair}
    \sigma = \frac{(N-\mathfrak{m})^{\beta}}{\sqrt{6}\pi} \quad\text{ and }\quad r = \frac{3\pi}{(N-\mathfrak{m})^{\alpha}}.
\end{equation}
There exists a universal constant $C>0$ such that for sufficiently large $N$,
\begin{equation*} 
    \|f-\mathcal{R}_{r,\sigma}f\|_{L^2([0,1];\mathbb{C}^m)} \leq C\widetilde{\mathcal{E}}(N,\mathfrak{m},\alpha,\beta)
\end{equation*}
holds, where
\begin{multline*} 
    \widetilde{\mathcal{E}}(N,\mathfrak{m},\alpha,\beta) \\
    := e^{-3\pi^2(N-\mathfrak{m})^{2(1-\beta)}}\max\bigg\{1,(N-\mathfrak{m})^{2\beta-1}\bigg\} + (N-\mathfrak{m})^{\alpha+\beta} e^{-3\pi^2(N-\mathfrak{m})^{2(\beta-\alpha)}}.
\end{multline*}
\end{theorem}

\subsection{Generalization using WNNs}\label{sec:WNNResult} 

Consider a graphon $W:[0,1]^2 \to [0,1]$. 
Suppressing the dependence on $(r,\sigma,N)$ for the ease of presentation, we define an (asymmetric) kernel $\mathcal{K}_j$, for $j=0,\cdots, 2N-1$, 
by
\begin{equation}\label{WNNkernel}
    \mathcal{K}_j(x,y) = \bigg(\frac{W(x,y)\chi_{\{x-r\leq j/2N}\}(x)}{\mathcal{W}_x}\bigg)\bigg(\frac{\,\mathrm{d}^2}{\,\mathrm{d} y^2}\, (s_N\,\mathcal{G}_{r,\sigma})(x-y)\bigg),
\end{equation}
where, for each $x\in [0,1]$, we set 
\begin{equation}\label{Wxbar}
    \mathcal{W}_x := \frac{1}{|\mathfrak{I}_x|}\int_{\mathfrak{I}_x} W(x,y)\,\mathrm{d}y \quad\text{ and }\quad \mathfrak{I}_x:= [0,1]\cap [-r+x, x+r],
\end{equation}
and for every $k\in\mathbb{N}$,
\begin{equation}\label{def:Gtruncder}
    \frac{\,\mathrm{d}^k}{\,\mathrm{d} x^k}\mathcal{G}_{r,\sigma}(x) = \mathcal{G}^{(k)}_{r,\sigma}(x) := \begin{cases}
                                    \mathcal{G}^{(k)}_{\sigma}(x) & -r\leq x\leq r,\\
                                    0 &\mathrm{otherwise}.
                                            \end{cases}
\end{equation}
Below is our result on the generalization capabilities of WNNs for Fourier bandlimited graphon signals, with a proof available in Section~\ref{sec:WNNthm}.

\begin{theorem} \label{thm:WNN}
Let $f\in \mathcal{B}_\mathfrak{m}$ such that $\|f\|_{L^2([0,1];\mathbb{C}^m)}= 1$. 
Let $W$ satisfy Assumption~\ref{assump:regular}. 
Let $\varepsilon\in (0,1)$. 
Let $\alpha=0.96$, $\beta=0.98$. 
Then there exist a universal constant $C>0$ and a ReLU WNN $\Psi_f$, using 
$\{\mathcal{K}_j\}_{j=0}^{2N-1}$ in \eqref{WNNkernel}, with $r,\sigma$ specified by
\begin{equation*} 
    r = \frac{3\pi}{(N-\mathfrak{m})^{\alpha}} \quad\text{ and }\quad \sigma = \frac{(N-\mathfrak{m})^{\beta}}{\sqrt{6}\pi},
\end{equation*}
such that for all $\varepsilon$ sufficiently small 
and for all $N=\lceil \varepsilon^{-10/9}\rceil\gg\mathfrak{m}$,
\begin{equation*} 
    \|\Psi_f-f\|_{L^2([r,1-r];\mathbb{C}^m)} \leq C\eta^{-1}K \varepsilon. 
\end{equation*}

More precisely, $\Psi_f$ assumes the structure described in Subsection~\ref{sec:WNNframework}, with two layers, including one filter layer, and $2N=2\lceil \varepsilon^{-10/9}\rceil$ nonunital linear weights supplied by the sampled functional values of $f$.
\end{theorem}

The specific choices of $\alpha,\beta$ are made for concreteness. In fact, by selecting $\alpha, \beta$ appropriately, we can obtain $N=\lceil \varepsilon^{-\delta}\rceil$, where $\delta$ is as close to $1$ as possible. 

Theorem~\ref{thm:WNN} relies on Theorem~\ref{thm:TsamplingregwG}, but they differ in predictive scope. 
While the latter guarantees predictions across $[0,1]$, the former provides guarantees for the specific {\it predictable} zone $[r,1-r]$, which, since $r=C(N-\mathfrak{m})^{-\alpha}$ and $N=\lceil\varepsilon^{-10/9}\rceil$, widens to $[0,1]$ as $\varepsilon\to 0^+$. 
This distinction arises from an implementation of an integration-by-parts result in Theorem~\ref{thm:WNN} 
(see Proposition~\ref{prop:WNNthmcont}), which secures its conclusion but in the process limits the prediction range.

\subsection{Generalization on graphs and GNN transferability}\label{sec:GNNthms}

Consider a graphon $W:[0,1]^2 \to [0,1]$. 
Let $\{G_n\}_{n = 1}^\infty$ be a sequence of simple graphs on $n$ vertices, generated from $W$ as in Subsection~\ref{sec:graphons}, and let $\{{\bf A}_n\}_{n=1}^{\infty}$ be the associated adjacency matrices. 
For each $G_n$, we define a family of graph filter kernels,
\begin{equation}\label{graphker}
    \mathcal{K}_{n,j}(x_k,x_l) = \frac{\chi_{\{x_k - r \leq j/2N\}}}{n}\bigg(\frac{\,\mathrm{d}^2}{\,\mathrm{d} x^2}\, (s_N\,\mathcal{G}_{r,\sigma})(x_k-x_l)\bigg)\bigg(\frac{[{\bf A}_n]_{k,l}}{\mathcal{W}_{x_k}}\bigg),
\end{equation}
for $j=0,\cdots, 2N-1$.
Here, $\mathcal{W}_{x_k}$ is as in \eqref{Wxbar} with $x=x_k$. 

The sequence $\{G_n\}_{n=1}^\infty$ can comprise either deterministic or random graphs. When we need to differentiate between the two cases, we adopt specific notations. If $G_n$ is deterministic, we denote it as $G^{\mathrm{det}}_n$, and $[{\bf A}_n]_{k,l}=W(x_k,x_l)$, for $k\not=l$. If $G_n$ is random, we represent it as $G^{\mathrm{ran}}_n$. In this case, ${\bf A}_n$ becomes a random matrix, and each $\mathcal{K}_{n,j}(x_k,x_l)$ is a random variable for $k \neq l$. A realization of $G^{\mathrm{ran}}_n$ yields a graph with a binary-entried adjacency matrix ${\bf A}_n$. In short, regardless of whether $G_n$ is deterministic or random or realized, we only need to substitute the corresponding adjacency matrix into \eqref{graphker}.

We now state our results on the generalization capabilities of GNNs for graphs belonging to the same graphon family. Their proofs can be found in Section~\ref{sec:GNNdet} and Section~\ref{sec:GNNran}, with respect to the order in which the theorems are presented.

\begin{theorem} (Deterministic GNNs) \label{thm:GNNdet}
Let $f\in \mathcal{B}_\mathfrak{m}$ be such that $\|f\|_{L^2([0,1];\mathbb{C}^m)}=1$. 
Let $W$ satisfy Assumption~\ref{assump:regular}. Let $\varepsilon\in (0,1)$. For $n\geq 1$, let $G^{\mathrm{det}}_n$ be a deterministic graph generated from $W$ and $f_n$ be an associated graph signal such that $f_n(x_k)=f(x_k)$ for all $x_k$. Then there exist a universal constant $C>0$ and a ReLU GNN $\Psi_{n,f}$, using the graph kernels $\{\mathcal{K}_{n,j}\}_{j=0}^{2N-1}$ in \eqref{graphker} with $r,\sigma$ specified by
\begin{equation*} 
    r = \frac{3\pi}{(N-\mathfrak{m})^\alpha} \quad\text{ and }\quad \sigma = \frac{(N-\mathfrak{m})^\beta}{\sqrt{6}\pi},
\end{equation*}
$\alpha=0.96$, $\beta=0.98$, such that for all $\varepsilon$ sufficiently small, all $N=\lceil \varepsilon^{-10/9}\rceil\gg\mathfrak{m}$, and all $n\geq\varepsilon^{-10/3}$, 
\begin{equation}\label{GNNdetconc}
    \|\overline{\Psi}_{n,f} - f\|_{L^1([r,1-r];\mathbb{C}^m)}\leq C\eta^{-2}(1+K)\,\varepsilon,
\end{equation}
where
\begin{equation}\label{def:barPsinf}
    \overline{\Psi}_{n,f}(x):=\sum_{k=1}^n \Psi_{n,f}(x_k)\chi_{I_k}(x), \quad\forall x\in [0,1]. 
\end{equation}

More precisely, $\Psi_{n,f}$ assumes the structure described in Subsection~\ref{sec:WNNframework}, with two layers, including one filter layer, and $2N=2\lceil \varepsilon^{-10/9}\rceil$ linear weights that are supplied by the sampled functional values of the graphon signal $f$.
\end{theorem} 


\begin{theorem} (Random GNNs) \label{thm:GNNran}
Let $f\in \mathcal{B}_\mathfrak{m}$ be such that $\|f\|_{L^2([0,1];\mathbb{C}^m)}=1$. 
Let $W$ satisfy Assumption~\ref{assump:regular}. 
Let $\varepsilon\in (0,1)$. For $n\geq 1$, let $G^{\mathrm{ran}}_n$ be a simple random graph generated from $W$ and $f_n$ be an associated graph signal such that $f_n(x_k)=f(x_k)$ for all $x_k$. Then there exist $C>0$ and a ReLU GNN $\Psi_{n,f}$, using the graph kernels $\{\mathcal{K}_{n,j}\}_{j=0}^{2N-1}$ in \eqref{graphker} with $r,\sigma$ specified by
\begin{equation*} 
    r = \frac{3\pi}{(N-\mathfrak{m})^\alpha} \quad\text{ and }\quad \sigma = \frac{(N-\mathfrak{m})^\beta}{\sqrt{6}\pi},
\end{equation*}
$\alpha=0.96$, $\beta=0.98$, such that for all $\varepsilon$ sufficiently small, all $N=\lceil \varepsilon^{-10/9}\rceil\gg\mathfrak{m}$, and all $n\geq\varepsilon^{-10/3}$, 
\begin{equation*}
    \|\overline{\Psi}_{n,f} - f\|_{L^1([r,1-r];\mathbb{C}^m)} \leq C\eta^{-2}(1+K)\,\varepsilon,
\end{equation*}
where $\overline{\Psi}_{n,f}$ is defined in \eqref{def:barPsinf}, with a probability at least
\begin{equation*}
    1 - 2n\varepsilon^{10(1-\alpha)/9}\exp\Big(-C\eta^2 n \varepsilon^{10(5-2\alpha)/9 + 2}\Big)
\end{equation*}

More precisely, $\Psi_{n,f}$ assumes the structure described in Subsection~\ref{sec:WNNframework}, with two layers, including one filter layer, and $2N=2\lceil \varepsilon^{-10/9}\rceil$ linear weights that are supplied by the sampled functional values of the graphon signal $f$.
\end{theorem} 



\subsection{Ramifications of the results}\label{sec:Ramifications}

Having presented our main results, we now highlight their most important ramifications.

\begin{enumerate}
\item The two major contributions of this work are 
presented in Theorems~\ref{thm:GNNdet} and \ref{thm:GNNran}. 
They provide a shallow GNN structure that uses a few generated graphon signal samples to reliably generalize on large graphs in the graphon family. As demonstrated in Theorem~\ref{thm:GNNdet}, it predicts accurately within an error of about $0.1$ for at least $1780$ vertices on graphs with at least $2150$ vertices, using a low-cost network with only $26$ linear weights. 
Moreover, the GNN is built with predetermined network parameters, requiring no learning process.

\item Our graphon approach achieves a dimensionless network structure, by utilizing one-dimensional input features, as explained in Remark~\ref{rem:inputfeat}, thus allowing us to circumvent the curse of dimensionality.

\item Our graphon approach allows seamless transferability of the network architecture across graphs of different sizes. This involves a straightforward substitution of the relevant adjacency matrix into the kernel \eqref{graphker} while maintaining comparable signal generalization estimates. The following corollary, given without proof, validates this notable outcome.

\begin{corollary} (GNN transferability) \label{cor:transfer}
Let $f\in \mathcal{B}_\mathfrak{m}$ be such that $\|f\|_{L^2([0,1];\mathbb{C}^m)}=1$. 
Let $W$ satisfy Assumption~\ref{assump:regular}. Let $\varepsilon\in (0,1)$ be sufficiently small. 
Let $G_{n_1}$, $G_{n_2}$ be two graphs (random or deterministic) generated from $W$ where $n_1\leq n_2$.  
Denote $\Psi_{n_i,f}$ as the ReLU GNN associated with $G_{n_i}$. Suppose that \eqref{GNNdetconc} holds for $\Psi_{n_1,f}$, in the case that $G_{n_1}$ is deterministic, or with a probability at least 
\begin{equation*}
    1 - 2n\varepsilon^{10(1-\alpha)/9}\exp\Big(-C\eta^2 n \varepsilon^{10(5-2\alpha)/9 + 2}\Big)
\end{equation*}
in the case that $G_{n_1}$ is random. Then conditional on whether $G_{n_2}$ is random or deterministic, the corresponding conclusion also applies to $\Psi_{n_2,f}$.
\end{corollary}

As an expository application, take the recommender systems that were used as motivation in the introduction 
of WNNs in \citep{ruiz2020graphon}. An $m$-dimensional graph signal encodes the ratings given by an individual user, represented as a vertex, for each of the $m$ products that the service recommends. 
Our approach constructs a GNN from ratings collected from $2N$ users, which then delivers predictions for $n \gg 2N$ users without requiring frequent model retraining as users join or leave, thus ensuring an adaptable and scalable recommender system.
Furthermore, we extend the methodology of \citep{ruiz2020graphon} by interpreting user correlations as the likelihood of shared tastes, facilitating the construction of random graphs. Our GNN maintains high prediction accuracy on these random graphs and subsequently enhances both the robustness and reliability of the recommender system.
\end{enumerate}

\section{Proof of Theorem~\ref{thm:TsamplingregwG}} \label{sec:Tsamplingthm}

In this section, we make the necessary steps to establish the validity of Theorem~\ref{thm:TsamplingregwG}. 
The main idea for the proof centers around the use of a truncated and regularized variant of Kluv\'anek's sampling formulation \citep{kluvanek1965sampling}. More precisely, let $\mathfrak{m}, N\in\mathbb{N}$, and suppose $N\geq\mathfrak{m}$. Let $f\in \mathcal{B}_\mathfrak{m}$. By Proposition~\ref{prop:gpsampling} we have
\begin{equation}\label{sampseriesrecall}
    f(x) = \sum_{j=0}^{2N-1} f\Big(\frac{j}{2N}\Big)s_N(x-j/2N), \quad \forall x\in [0,1],
\end{equation}
where $s_N$ is as in \eqref{def:sampf}. Our goal is to replace the series on the right-hand side of \eqref{sampseriesrecall} with 
\begin{equation*}
    \sum_{j=0}^{2N-1} f\Big(\frac{j}{2N}\Big)\mathcal{G}(x-j/2N)
\end{equation*}
where $\mathcal{G}$ is some function exhibiting a ``Gaussian-like'' behavior but is only supported in a vicinity of $x$ of some appropriate radius $r$. To begin, we need to establish a general result for regularized sampling. In this section as well as others that follow, we will make repeated use of the McLaurin-Cauchy integral test \citep[\S 3.3, Theorem 3]{knopp1956infinite} and the Mills' ratio formula \citep{boyd1959inequalities}, which we cite here for the reader's convenience:
\begin{equation}\label{Mill}
    \int_x^\infty e^{-t^2/2}\,\mathrm{d}t \leq \frac{\pi e^{-x^2/2}}{\sqrt{(\pi-2)^2x^2+2\pi} + 2x}, \quad\forall x>0.
\end{equation}

\subsection{A regularized sampling result}

Let $\phi:[0,1] \to\mathbb{C}$ be a continuous function such that $\phi(0)=\phi(1) = 1$ and its Fourier coefficients are absolutely summable, i.e. $\hat{\phi}\in \ell^1(\mathbb{Z})$. Then $\hat{\phi}\in \ell^2(\mathbb{Z})$ as well, and from the Fourier inversion theorem, 
we have $\phi\in L^2([0,1])$ with
\begin{equation}\label{phi}
    \sum_{l\in\mathbb{Z}} \hat{\phi}(l) = \phi(0) = 1.
\end{equation}
For $N\in\mathbb{N}$, let $\chi_{B_N}$ denote a function that is one on $B_N = \{-N,\dots,N-1\} \subset \mathbb{Z}$ and zero on $\mathbb{Z}\setminus B_N$. For each $k\in\mathbb{Z}$, we define
\begin{equation}\label{def:mu}
    \mu(k) : =\chi_{B_N}\ast\hat{\phi}(k) = \sum_{l\in\mathbb{Z}} \chi_{B_N}(k-l)\hat{\phi}(l) = \sum_{l\in k-B_N}\hat{\phi}(l),
\end{equation}
and
\begin{equation}\label{def:nu}
    \nu(k) := \sum_{l\not\in k-B_N}\hat{\phi}(l).
\end{equation}
We gather from \eqref{phi}, \eqref{def:mu}, \eqref{def:nu} that
\begin{equation}\label{nu&mu1}
   \sum_{l\in\mathbb{Z}} \mu(k+l2N) = \sum_{l\in\mathbb{Z}} \hat{\phi}(l)= 1 = \nu(k) + \mu(k), \quad\forall k\in\mathbb{Z},
\end{equation}
and so we deduce
\begin{equation}\label{nu&mu2}
    \nu(k) = \sum_{l\in\mathbb{Z}\setminus\{0\}} \mu(k+l2N), \quad\forall k\in\mathbb{Z}.
\end{equation}

Let $N\geq\mathfrak{m}$, and let $f\in\mathcal{B}_\mathfrak{m}$. We define
\begin{equation}\label{def:reg}
    \mathcal{R}_\phi f(x) := \sum_{j=0}^{2N-1} f\Big(\frac{j}{2N}\Big)s_N(x-j/2N)\phi(x-j/2N), \quad \forall x\in [0,1],
\end{equation}
where, from definition \eqref{def:sampf}, we can write 
\begin{equation*} 
    \hat{s}_N(k) =\frac{1}{2N}\chi_{B_N}(k), \quad\forall k\in\mathbb{Z}.
\end{equation*}
Note that all the functions involved in \eqref{def:reg} are continuous on $\mathbb{T}\cong [0,1)$.  
Reasonably, we can no longer expect $f=\mathcal{R}_\phi f$. However, we claim that the regularized sampling error, 
$f-\mathcal{R}_\phi f$, can be understood in terms of $\mu, \nu$ in \eqref{def:mu}, \eqref{def:nu}, respectively.

\begin{lemma} \label{lem:regerr}
Let $N\geq\mathfrak{m}$, and let $f\in \mathcal{B}_\mathfrak{m}$. Then
\begin{equation*} 
    \hat{f}(k)-\widehat{\mathcal{R}_\phi f}(k) = \begin{cases}
                                \hat{f}(k)\nu(k), & k\in B_\mathfrak{m}\\
                                -\hat{f}(k-2lN)\mu(k), & \exists l\neq 0\ \mathrm{s.t}\ k-2lN\in B_\mathfrak{m}\\
                                0, &\mathrm{ otherwise}
                             \end{cases}.
\end{equation*}
\end{lemma}

\begin{proof}
Recall from Lemma~\ref{lem:tech} and its proof that
\begin{equation}\label{2FTrecall}
    \hat{f}(k) = \begin{cases} 
                    \frac{1}{2N}\sum_{j=0}^{2N-1}f\Big(\frac{j}{2N}\Big)e^{-i2\pi kj/2N}, &k\in B_N\\
                    0, &\mathrm{ otherwise}
                 \end{cases}.
\end{equation}
Particularly, if $k\in B_N\setminus B_\mathfrak{m}$,
\begin{equation} \label{eq:flat}
    \hat{f}(k) = \frac{1}{2N}\sum_{j=0}^{2N-1}f\Big(\frac{j}{2N}\Big)e^{-i2\pi kj/2N} = 0.
\end{equation}
Let $k\in B_N$. By combining \eqref{2FTrecall} with \eqref{def:mu}, \eqref{def:reg}, and definition \eqref{def:sampf}, we obtain
\begin{equation*}
    \widehat{\mathcal{R}_\phi f}(k) = \frac{1}{2N}\sum_{j=0}^{2N-1} f(j/N)e^{-i2\pi kj/2N}\mu(k) = \hat{f}(k)\mu(k).
\end{equation*}
Therefore, from \eqref{nu&mu1}
\begin{equation}\label{regerrlem1}
    \hat{f}(k) - \widehat{\mathcal{R}_\phi f}(k)= \hat{f}(k)\nu(k).
\end{equation}
Further, from \eqref{eq:flat}, $\hat{f}(k) - \widehat{\mathcal{R}_\phi f}(k)= 0$, if $k\in B_N\setminus B_\mathfrak{m}$. 

Now suppose $k\not\in B_N$ but $k-2lN\in B_N$ for some $l\neq 0$. Then
\begin{equation}\label{regerrlem2}
    \begin{split}
        \hat{f}(k) - \widehat{\mathcal{R}_\phi f}(k) &= - \widehat{\mathcal{R}_\phi f}(k) \\
        &= -\frac{1}{2N}\sum_{j=0}^{2N-1} f(j/N)e^{-i2\pi (k-2lN)j/2N}\mu(k) = -\hat{f}(k-2lN)\mu(k),   
    \end{split}
\end{equation}
where we have again used \eqref{2FTrecall}. In the case $k-l2N\in B_N\setminus B_\mathfrak{m}$, it follows from \eqref{eq:flat} and \eqref{regerrlem2} that
\begin{equation}\label{regerrlem3}
    \hat{f}(k)-\widehat{\mathcal{R}_\phi f}(k) = 0.
\end{equation}
Combining \eqref{eq:flat}, \eqref{regerrlem1}, \eqref{regerrlem2}, \eqref{regerrlem3}, we conclude the lemma. 
\end{proof} 

We now provide the bounds on the difference between $f$ and $\mathcal{R}_\phi f$ in both the $L^2$ and $L^\infty$ norms, which will eventually be employed to prove Theorem~\ref{thm:TsamplingregwG} in 
Subsection~\ref{sec:truncregwG}.

\begin{lemma} \label{lem:regsamp} 
Let $N\geq\mathfrak{m}$, and let $f\in \mathcal{B}_\mathfrak{m}$. Then 
\begin{equation}\label{l2replace}
    \overline{\varepsilon}_1\|f\|_{L^2([0,1];\mathbb{C}^m)}\leq \|f-\mathcal{R}_\phi f\|_{L^2([0,1];\mathbb{C}^m)} \leq\,\overline{\varepsilon}_2\|f\|_{L^2([0,1];\mathbb{C}^m)},
\end{equation}
where
\begin{equation*}
    \begin{split}
        \overline{\varepsilon}_1 &:= \min_{k\in B_\mathfrak{m}}\bigg( |\nu(k)|^2 +  \sum_{l\in\mathbb{Z}\setminus\{0\}} |\mu(k+l2N)|^2 \bigg)^{1/2}\\
        \overline{\varepsilon}_2 &:= \max_{k\in B_\mathfrak{m}} \bigg(|\nu(k)|^2 +  \sum_{l\in\mathbb{Z}\setminus\{0\}} |\mu(k+l2N)|^2 \bigg)^{1/2}.
    \end{split}
\end{equation*}
\end{lemma}

\begin{proof}
By applying Lemma~\ref{lem:regerr} and utilizing Plancherel's theorem, we obtain
\begin{equation}\label{L2bdequal}
    \|f-\mathcal{R}_\phi f\|_{L^2([0,1];\mathbb{C}^m)}^2 = \sum_{k\in B_\mathfrak{m}} |\hat{f}(k)\nu(k)|^2 + \sum_{l\in\mathbb{Z}\setminus\{0\}} \sum_{k\in B_\mathfrak{m}} |\hat{f}(k)\mu(k+2l\pi)|^2,
\end{equation}
from which \eqref{l2replace} automatically follows. We note that the $L^2$-bounds cannot be improved, due to the equality \eqref{L2bdequal}. 
\end{proof} 

\subsection{Sampling with $\mathcal{G}_{r,\sigma}$} \label{sec:truncregwG}

In this subsection we will apply Lemma~\ref{lem:regsamp} to specific regularizers $\phi$, yielding the proof of Theorem~\ref{thm:TsamplingregwG}.

Let $\sigma>0$ and recall the regularizer
\begin{equation}\label{def:G}
    \mathcal{G}_{\sigma}(x) = c(\sigma)\sum_{k\in\mathbb{Z}} e^{-k^2\sigma^{-2}/2}e^{i2\pi kx} = 2c(\sigma)\sum_{k = 0}^\infty e^{-k^2\sigma^{-2}/2}\cos (2\pi kx), \quad\forall x\in [0,1].
\end{equation}
presented previously in \eqref{def:Grsigma}. Since $\mathcal{G}_\sigma(0)=\mathcal{G}_\sigma(1)$, we consider $\mathcal{G}_\sigma$ as a function on $\mathbb{T}\cong [0,1)\cong [-1/2,1/2)$. The parameter $\sigma$ is a variance parameter, and $c(\sigma)>0$ is a normalization constant guaranteeing that
\begin{equation*}
    \mathcal{G}_\sigma(0) = \|\hat{\mathcal{G}_\sigma}\|_{\ell^1(\mathbb{Z})} = c(\sigma)\sum_{k\in\mathbb{Z}} e^{-k^2\sigma^{-2}/2}=1,
\end{equation*}
where $\hat{\mathcal{G}}_\sigma(k)=c(\sigma)e^{-k^2\sigma^{-2}/2}$ according to \eqref{def:G}. Note that $c(\sigma)\leq 1$. 

The motivation for the definition \eqref{def:G} stems from the fact that $\mathcal{G}_\sigma$ mimics ``Gaussian''-like behavior. We highlight this with the following lemma. 

\begin{lemma}\label{lem:GaussianLike}
The function $\mathcal{G}_\sigma$, given in \eqref{def:G}, is even and satisfies
\begin{equation}\label{posclaim}
    \mathcal{G}_\sigma(x)> 0, \quad \forall x\in\mathbb{T}.
\end{equation}
Moreover, 
\begin{equation}\label{Poisson_app}
    \mathcal{G}_\sigma(x)= \sum_{l\in\mathbb{Z}} \mathfrak{g}(x+l), \quad \forall x\in\mathbb{T}.
\end{equation}
where $\mathfrak{g}(\xi):= d(\sigma)e^{-2\xi^2\pi^2\sigma^2}$, with $d(\sigma):=\sigma c(\sigma)\sqrt{2\pi}$. 
\end{lemma}

\begin{proof}
    For any $g\in L^2(\mathbb{R})$ we denote the Fourier transform and its inverse by
\begin{equation*}
    \mathcal{F}_{\mathbb{R}}g(\xi) := \int_{\mathbb{R}} g(u)e^{-i2\pi u\xi}\,\mathrm{d}u \quad\text{ and }\quad \mathcal{F}^{-1}_{\mathbb{R}}g(\xi) := \int_{\mathbb{R}} g(u)e^{i2\pi u\xi}\,\mathrm{d}u.
\end{equation*}
Precisely, it is known that $\mathcal{F}_{\mathbb{R}}\circ\mathcal{F}^{-1}_{\mathbb{R}}={\rm Id}_{L^2(\mathbb{R})} = \mathcal{F}^{-1}_{\mathbb{R}}\circ\mathcal{F}_{\mathbb{R}}$ \citep[Theorems~8.26 and 8.29]{folland1999real}. Now, the function $\mathfrak{g}^*(\xi) := c(\sigma)e^{-\xi^2\sigma^{-2}/2}$ is such that $\mathfrak{g}^*(l) = \hat{\mathcal{G}_\sigma}(l)$ for all $l\in\mathbb{Z}$, and it is the Fourier transform of
\begin{equation*}
    \mathcal{F}^{-1}_{\mathbb{R}}\mathfrak{g}^*(\xi)=\mathfrak{g}(\xi)=d(\sigma)e^{-2\xi^2\pi^2\sigma^2}, \quad \forall\xi\in\mathbb{R},
\end{equation*}
where $d(\sigma)=\sigma c(\sigma)\sqrt{2\pi}$. Therefore, by invoking the Poisson summation formula \citep[Theorem~8.32]{folland1999real}, we obtain
\begin{equation*}
    \mathcal{G}_\sigma(x)=\sum_{l\in\mathbb{Z}} \hat{\mathcal{G}_\sigma}(l) e^{i2\pi lx} = \sum_{l\in\mathbb{Z}} \mathfrak{g}^*(l)e^{i2\pi lx} = \sum_{l\in\mathbb{Z}} \mathfrak{g}(x+l)>0, \quad \forall x\in\mathbb{T},
\end{equation*} 
which proves both \eqref{posclaim} and \eqref{Poisson_app}. Finally, that $\mathcal{G}_\sigma$ is even can be seen from definition \eqref{def:G}. 
\end{proof} 

Next, we recall a truncated version of $\mathcal{G}_\sigma$, also presented earlier in \eqref{def:Grsigma}. Let $r\in (0,1/2)$ and consider, for $x\in [-1/2, 1/2)\cong\mathbb{T}$, 
\begin{equation*} 
    \mathcal{G}_{r,\sigma}(x) = \begin{cases}
                        \mathcal{G}_\sigma(x) = c(\sigma)\sum_{k\in \mathbb{Z}} e^{-k^2\sigma^{-2}/2}e^{i2\pi kx}, & -r \leq x\leq r\\
                        0, &\text{ otherwise}
                      \end{cases}.
\end{equation*}
Note that $\mathcal{G}_{r,\sigma} = \mathcal{G}_\sigma\chi_{I_r}$, where $\chi_{I_r}$ is the indicator function on the interval $I_r = [-r,r]$. Since $\mathcal{G}_\sigma$ is even and real, so is $\hat{\mathcal{G}}_{r,\sigma}$ on $\mathbb{Z}$. 
Recall the following truncated regularized sampling series of $f$ \eqref{reconstruction}:
\begin{equation*}
    \mathcal{R}_{r,\sigma}f(x) = \sum_{j=0}^{2N-1} f\Big(\frac{j}{2N}\Big)s_N(x-j/2N)\mathcal{G}_{r,\sigma}(x-j/2N).
\end{equation*}
We compute that
\begin{equation*}
    \widehat{\mathcal{R}_{r,\sigma}f}(k) = \frac{1}{2N}\sum_{j=0}^{2N-1}f\Big(\frac{j}{2N}\Big)e^{-i2\pi kj/2N}\mu_{r,\sigma}(k)
\end{equation*}
where, we have defined 
\begin{equation}\label{def:mur}
    \mu_{r,\sigma}(k) := \chi_{B_N}\ast\hat{\mathcal{G}}_{r,\sigma}(k) = \sum_{l\in k-B_N} \hat{\mathcal{G}}_{r,\sigma}(l).
\end{equation}
If we let
\begin{equation}\label{def:nursigma}
    \nu_{r,\sigma}(k) := \sum_{l\not\in k-B_N} \hat{\mathcal{G}}_{r,\sigma}(l),
\end{equation}
then, similarly to \eqref{nu&mu1}, \eqref{nu&mu2}, we gather that
\begin{equation*} \label{nur&mur}
    \nu_{r,\sigma}(k) + \mu_{r,\sigma}(k) = \sum_{l\in\mathbb{Z}} \hat{\mathcal{G}}_{r,\sigma}(l) = 1, \quad\text{ and }\quad
    \nu_{r,\sigma}(k) = \sum_{l\in\mathbb{Z}\setminus\{0\}} \mu_{r,\sigma}(k+l2N), \quad\forall k\in\mathbb{Z}. 
\end{equation*}
Both $\nu_{r,\sigma}, \mu_{r,\sigma}$ are real and even on $\mathbb{Z}$. It should be clear from the past analysis that $\mu_{r,\sigma}$ will play the role of $\mu$ in \eqref{def:mu}. Therefore, to access the upper bound in \eqref{l2replace}, 
we would need an estimate for
\begin{equation*}
    \max_{k\in B_\mathfrak{m}}\bigg(|\nu_{r,\sigma}(k)|^2 + \sum_{l\in\mathbb{Z}\setminus\{0\}} |\mu_{r,\sigma}(k+l2N)|^2\bigg)^{1/2}.
\end{equation*}
For ease of computation, we will replace this quantity with
\begin{equation}\label{l2replace*}
    \max_{k\in B_\mathfrak{m}}|\nu_{r,\sigma}(k)| +  \max_{k\in B_\mathfrak{m}}\bigg(\sum_{l\in\mathbb{Z}\setminus\{0\}} |\mu_{r,\sigma}(k+l2N)|^2\bigg)^{1/2}.
\end{equation}
We first handle $\max_{k\in B_\mathfrak{m}}|\nu_{r,\sigma}(k)|$. Referring to definition \eqref{def:nursigma}, observe that for any $k\in\mathbb{Z}$, 
\begin{align}
    \nonumber \nu_{r,\sigma}(k) &= 1 - \sum_{l\in k-B_N} \hat{\mathcal{G}}_{r,\sigma}(l)\\ 
    \nonumber &= 1 - \sum_{l\in k-B_N} \bigg(\hat{\mathcal{G}}_\sigma(l) - \int_{u\not\in [-r,r]\cap [-1/2,1/2)} \mathcal{G}_\sigma (u)e^{-i2\pi lu}\,\mathrm{d}u\bigg)\\
    \nonumber &= \sum_{l\not\in k-B_N}\hat{\mathcal{G}}_\sigma(l) + \int_{u\not\in [-r,r]\cap [-1/2,1/2)} \mathcal{G}_\sigma(u)\sum_{l\in k-B_N} e^{-i2\pi lu}\,\mathrm{d}u\\
    \label{explicitnur} &= \sum_{l\not\in k-B_N}\hat{\mathcal{G}}_\sigma(l) + \int_{u\not\in [-r,r]\cap [-1/2,1/2)} \mathcal{G}_\sigma(u)e^{i2\pi(k-N+1)u}\bigg(\frac{1-e^{-i4\pi Nu}}{1-e^{-i2\pi u}}\bigg)\,\mathrm{d}u.
\end{align}
Note that the first sum on the right-hand side above is $\nu(k)$ in \eqref{def:nu} with $\phi=\mathcal{G}_\sigma$. This is the {\it regularization} error. The integral on the right-hand side, carrying an $r$ in its expression, is the {\it truncation} error. 

A quick calculation from \eqref{explicitnur} gives
\begin{equation}\label{maxerr}
    \max_{k\in B_{\mathfrak{m}}} |\nu_{r,\sigma}(k)| \leq \sum_{l\not\in \mathfrak{m}-1-B_N}\hat{\mathcal{G}_\sigma}(l) + \sup_{x\in [r,1/2)} \frac{2\mathcal{G}_\sigma(x)}{|1-e^{-i2\pi r}|}.
\end{equation}
Hence, a bound on $\max_{k\in B_{\mathfrak{m}}} |\nu_{r,\sigma}(k)|$ can be derived from an upper bound for the right-hand side in \eqref{maxerr}.

\begin{proposition}\label{prop:nu}
Let $\nu_{r,\sigma}$ be as in \eqref{def:nursigma} and $r,\sigma$ be specified as in \eqref{rsigmapair}. Then 
\begin{equation*}
    \max_{k\in B_{\mathfrak{m}}} |\nu_{r,\sigma}(k)| \leq C\widetilde{\mathcal{E}}_1(N,\mathfrak{m},\alpha,\beta),
\end{equation*}
for some $C>0$, and
\begin{equation*}
    \widetilde{\mathcal{E}}_1(N,\mathfrak{m},\alpha,\beta) := e^{-3\pi^2(N-\mathfrak{m})^{2(1-\beta)}} \max\bigg\{1,(N-\mathfrak{m})^{2\beta-1}\bigg\} + (N-\mathfrak{m})^{\alpha+\beta} e^{-3\pi^2(N-\mathfrak{m})^{2(\beta-\alpha)}}.
\end{equation*}
\end{proposition}

\begin{proof}
Since $e^{-x^2\sigma^{-2}/2}$ is even on $\mathbb{R}$ and strictly decreasing on $[0,\infty)$, we have
\begin{equation}\label{truncregerrbound1}
    \sum_{l\not\in \mathfrak{m}-1-B_N} \hat{\mathcal{G}_\sigma}(l) \leq 2\sum_{l\geq N-\mathfrak{m}}\hat{\mathcal{G}_\sigma} \leq 2\hat{\mathcal{G}_\sigma}(N-\mathfrak{m}) + 2\int_{N-\mathfrak{m}}^{\infty} e^{-u^2\sigma^{-2}/2}\,\mathrm{d}u,
\end{equation}   
and so using Mills' ratio \eqref{Mill} we obtain 
\begin{equation}\label{Mills}
    \int_{N-\mathfrak{m}}^{\infty} e^{-u^2\sigma^{-2}/2}\,\mathrm{d}u \leq \frac{\pi\sigma e^{-(N-\mathfrak{m})^2\sigma^{-2}/2}}{\sqrt{(\pi-2)^2 (N-\mathfrak{m})^2\sigma^{-2}+2\pi} + 2(N-\mathfrak{m})\sigma^{-1}}.
\end{equation}
By combining \eqref{truncregerrbound1} and \eqref{Mills}, and performing the necessary simplifications, we arrive at
\begin{equation*}
    \sum_{l\not\in \mathfrak{m}-1-B_N} \hat{\mathcal{G}}_{\sigma}(l)\leq 4e^{-(N-\mathfrak{m})^2\sigma^{-2}/2}\max\bigg\{1,\frac{\sigma^2}{N-\mathfrak{m}}\bigg\}.
\end{equation*}
Keeping in mind that $\sigma = (N-\mathfrak{m})^{\beta}/(\sqrt{6}\pi)$ in \eqref{rsigmapair}, we obtain the estimate
\begin{equation}\label{truncregerrbound2}
    \sum_{l\not\in \mathfrak{m}-1-B_N} \hat{\mathcal{G}}_{\sigma}(l)\leq Ce^{-3\pi^2(N-\mathfrak{m})^{2(1-\beta)}}\max\bigg\{1,(N-\mathfrak{m})^{2\beta-1}\bigg\},
\end{equation}
for some $C>0$.

Continuing, since $r\in (0,1/2)$, the function $e^{-2(x+r)^2\pi^2\sigma^{2}}$ is strictly decreasing on $[0,\infty)$ and increasing on $(-\infty,-1]$. Thus, we deduce from \eqref{Poisson_app} and Mills' ratio that for $x\in [r,1/2)$,
\begin{equation*}
    \begin{split}
        \mathcal{G}_\sigma(x) &= \sum_{l\in\mathbb{Z}} \mathfrak{g}(x+l) \leq 3d(\sigma)e^{-2x^2\pi^2\sigma^2} + 2d(\sigma)\int_{1-x}^\infty e^{-2\pi^2\sigma^2u^2} \,\mathrm{d}u \\
        &\leq 3\sigma\sqrt{2\pi} e^{-2x^2\pi^2\sigma^2} + \frac{\sqrt{\pi} e^{-2\pi^2\sigma^2(1-x)^2}}{\sqrt{(\pi-2)^2 \pi^2\sigma^2(1-x)^2+\pi} + 2\pi\sigma(1-x)},
    \end{split}
\end{equation*}
which, by the selections of $r, \sigma$ in \eqref{rsigmapair}, simplifies to
\begin{equation} \label{G(r)}
    \mathcal{G}_\sigma(x) \leq C(N-\mathfrak{m})^{\beta} e^{-3\pi^2(N-\mathfrak{m})^{2(\beta-\alpha)}}
\end{equation}
for some $C>0$. Since $r=3\pi(N-\mathfrak{m})^{-\alpha}$, we can ensure that when $N$ is sufficiently large, 
\begin{equation}\label{secondbd}
    \begin{split}
        \sup_{x\in [r,1/2)} \frac{\mathcal{G}_\sigma(x)}{|1-e^{-i2\pi r}|} &\leq \frac{C(N-\mathfrak{m})^{\beta}}{r} \bigg(e^{-3\pi^2(N-\mathfrak{m})^{2(\beta-\alpha)}}\bigg)\\
        &\leq C(N-\mathfrak{m})^{\alpha+\beta} e^{-3\pi^2(N-\mathfrak{m})^{2(\beta-\alpha)}}.
    \end{split}
\end{equation}
By integrating \eqref{maxerr}, \eqref{truncregerrbound2} and \eqref{secondbd}, we acquire the conclusion of the proposition.
\end{proof}

We now shift our focus to the term involving $\mu_{r,\sigma}$ in \eqref{l2replace*}.

\begin{proposition} \label{prop:mu}
Let $\mu_{r,\sigma}$ be as in \eqref{def:mur} and $r,\sigma$ be specified as in \eqref{rsigmapair}. There exists a universal constant $C>0$ such that
\begin{equation}\label{mul2}
    \max_{k\in B_\mathfrak{m}} \bigg(\sum_{l\in\mathbb{Z}\setminus\{0\}} |\mu_{r,\sigma}(k+l2N)|^2\bigg)^{1/2} \leq C\widetilde{\mathcal{E}}_2(N,\mathfrak{m},\alpha,\beta),
\end{equation}
where
\begin{equation*}
    \widetilde{\mathcal{E}}_2(N,\mathfrak{m},\alpha,\beta) := (N-\mathfrak{m})^\beta e^{-3\pi^2(N-\mathfrak{m})^{2(\beta-\alpha)}} + (N-\mathfrak{m})^{2\beta-\alpha-1} e^{-3\pi^2(N-\mathfrak{m})^{2(1-\beta)}}.
\end{equation*}
\end{proposition}

\begin{proof}
Since $\mu_{r,\sigma}$ is even on $\mathbb{Z}$ and $N>\mathfrak{m}$, it is enough to provide analysis for the convergence in \eqref{mul2} for all integers of the form $k'=k+2l'N> 0$, where $l'\in\mathbb{N}$, and so $k'-B_N$ consists of only positive integers. Fix one such $k'$. Then, by definition \eqref{def:mur}
\begin{equation}\label{murint}
    \mu_{r,\sigma}(k+2l'N) = \mu_{r,\sigma}(k') = \sum_{l\in k'-B_N} \hat{\mathcal{G}}_{r,\sigma}(l) = \sum_{l\in k'-B_N} \int_{-r}^{r} \mathcal{G}_\sigma(u) e^{-i2\pi lu}\,\mathrm{d}u.
\end{equation}
From the Poisson formula \eqref{Poisson_app}, we obtain for each $l\in k'-B_N$ in \eqref{murint} that
\begin{equation}\label{Ijl}
    \begin{split}
        \int_{-r}^{r} \mathcal{G}_\sigma(u) e^{-i2\pi lu}\,\mathrm{d}u &= d(\sigma) \sum_{j\in\mathbb{Z}} \int_{-r}^{r} e^{-2(u+j)^2\pi^2\sigma^2}e^{-i2\pi lu}\,\mathrm{d}u\\
        &= d(\sigma) \sum_{j\in\mathbb{Z}} \int_{j-r}^{j+r} e^{-2u^2\pi^2\sigma^2} e^{-i2\pi lu}\,\mathrm{d}u\\
        &= d(\sigma) e^{-l^2/(2\sigma^2)} \sum_{j\in\mathbb{Z}} \int_{j-r}^{j+r} e^{-2\pi^2\sigma^2(u + il/(2\pi\sigma^2))^2}\,\mathrm{d}u \\
        &=: d(\sigma) \sum_{j\in\mathbb{Z}} \mathcal{I}_{j,l}.
    \end{split}
\end{equation}
Fix one $j\in\mathbb{Z}$ and one $l\in k'-B_N$. We consider the contour integration of the holomorphic function $e^{-l^2/(2\sigma^2)}e^{-2\pi^2\sigma^2 z^2}$ over a rectangle consisting of the following four lines on the complex plane:
\begin{equation}\label{rlines}
    \begin{split}
        \text{vertical:}\quad \{(j+r,y): 0\leq y\leq l/(2\pi\sigma^2)\}, \quad &\{(j-r,y): 0\leq y\leq l/(2\pi\sigma^2)\}\\
        \text{horizontal:}\quad \{(x,0): j-r\leq x\leq j+r\}, \quad &\{(x,l/(2\pi\sigma^2)): j-r\leq x\leq j+r\}.
    \end{split}
\end{equation}
Cauchy's integral theorem \citep[Theorem 2.2]{stein2010complex} then gives that
\begin{equation}\label{contint}
    \begin{split}
        0 = & -\mathcal{I}_{j,l} + e^{-l^2/(2\sigma^2)} \int_{j-r}^{j+r} e^{-2\pi^2\sigma^2 x^2} \,\mathrm{d}x \\
        & + e^{-l^2/(2\sigma^2)}\bigg(\int_{l/(2\pi\sigma^2)}^0 e^{-2\pi^2\sigma^2(j-r +iy)^2}\,i\mathrm{d}y + \int_0^{l/(2\pi\sigma^2)} e^{-2\pi^2\sigma^2(j+r+iy)^2} \,i\mathrm{d}y\bigg).
    \end{split}
\end{equation}
The contribution of the first vertical line in \eqref{rlines} is
\begin{align} \label{vertical} 
    e^{-l^2/(2\sigma^2)}\bigg| \int_0^{l/(2\pi\sigma^2)} e^{-2\pi^2\sigma^2(j+r+iy)^2}\,i\mathrm{d}y\bigg| \leq &\frac{e^{-2\pi^2\sigma^2(j+r)^2}}{\sqrt{2}\pi\sigma}\int_0^{l/(\sqrt{2}\sigma)} e^{y^2-l^2/(2\sigma^2)}\,\mathrm{d}y\\
    \nonumber &=\frac{e^{-2\pi^2\sigma^2(j+r)^2}}{\sqrt{2}\pi\sigma}\int_0^{l/(\sqrt{2}\sigma)} e^{(y-l/(\sqrt{2}\sigma))(y+l/(\sqrt{2}\sigma))}\,\mathrm{d}y.
\end{align}
We note that the second and third integrals appearing above take the form $e^{-a^2}\int_0^a e^{-u^2}\,\mathrm{d}u$, which is a {\it Dawson's integral} \citep{carneiro2013bandlimited}. To handle it, we make a convenient change of variables as follows. Let $u\geq 0$ be such that $-u=y-l/(\sqrt{2}\sigma)$. Then since $0\leq y\leq l/(\sqrt{2}\sigma)$,
\begin{equation*}
    -\frac{2l}{\sqrt{2}\sigma} u\leq y^2-\frac{l^2}{2\sigma^2} = -u(y+\frac{l}{\sqrt{2}\sigma})\leq -\frac{l}{\sqrt{2}\sigma} u.
\end{equation*}
Putting this back into \eqref{vertical} gives us
\begin{equation*}
    \begin{split}
        e^{-l^2/(2\sigma^2)}\bigg| \int_0^{l/(2\pi\sigma^2)} e^{-2\pi^2\sigma^2(j+r+iy)^2}\,i\mathrm{d}y\bigg| &\leq \frac{e^{-2\pi^2\sigma^2(j+r)^2}}{\sqrt{2}\pi\sigma}\int_0^{l/(\sqrt{2}\sigma)} e^{-ul/(\sqrt{2}\sigma)}\,\mathrm{d}u\\
        &\leq \frac{e^{-2\pi^2\sigma^2(j+r)^2}}{l\pi}\int_0^{l^2/(2\sigma^2)} e^{-u}\,\mathrm{d}u\\
        &=: \frac{e^{-2\pi^2\sigma^2(j+r)^2}\mathcal{J}_l}{l\pi},
    \end{split}
\end{equation*}
where, $\mathcal{J}_l\to 1$ as $l\to\infty$. By the Poisson summation formula \eqref{Poisson_app} again,
\begin{equation*}
        \sum_{l\in k'-B_N} \frac{\mathcal{J}_l}{l\pi}\sum_{j\in\mathbb{Z}} d(\sigma) e^{-2\pi^2\sigma^2(j+r)^2} = 
        \sum_{l\in k+2l'N-B_N} \frac{\mathcal{G}_\sigma(r)\mathcal{J}_l}{l\pi}.
\end{equation*}
We square-sum the final term above over $l' \in \mathbb{N}$, acquiring for any $k\in B_\mathfrak{m}$
\begin{equation}\label{sideint}
    \begin{split}
        \sum_{l'\in\mathbb{N}} \bigg| \sum_{l\in k+2l'N-B_N} \frac{\mathcal{G}_\sigma(r)\mathcal{J}_l}{l\pi}\bigg|^2 &\leq \frac{2N\mathcal{G}_\sigma(r)^2}{\pi^2} \sum_{l'=1}^{\infty} \quad \sum_{l=k+(2l'-1)N+1}^{k+(2l'+1)N} \frac{1}{l^2}\\
        &= \frac{2N\mathcal{G}_\sigma(r)^2}{\pi^2} \sum_{l=k+N+1}^\infty \frac{1}{l^2}\\
        &\leq \frac{2N\mathcal{G}_\sigma(r)^2}{\pi^2}\int_{k+N}^\infty \frac{1}{l^2}\\
        &\leq \frac{2N\mathcal{G}_\sigma(r)^2}{\pi^2(N-\mathfrak{m})}.
    \end{split}
\end{equation}
Employing the same method for the contribution of the second vertical line in \eqref{rlines}, and keeping in mind that $\mathcal{G}_\sigma(r) = \mathcal{G}_\sigma(-r)$, we also arrive at an analogous bound. 
Finally, a consideration of the bottom horizontal integral in \eqref{contint} leads to:
\begin{equation*}
    e^{-l^2/(2\sigma^2)} d(\sigma) \sum_{j\in\mathbb{Z}} \int_{j-r}^{j+r} e^{-2\pi^2\sigma^2x^2}\,\mathrm{d}x = e^{-l^2/(2\sigma^2)}\int_{-r}^r \mathcal{G}_\sigma(x) \,\mathrm{d}x \leq \frac{Ce^{-l^2/(2\sigma^2)}}{(N-\mathfrak{m})^\alpha}.
\end{equation*}
Similarly as before, since $l \in k' - B_N$, where $k' = k + 2l'N$, for $k \in B_\mathfrak{m}$ and $l' \in \mathbb{N}$, summing the last term above over $l\in k' - B_N$ and then over $l' \in \mathbb{N}$ results in 
\begin{equation*}
    \begin{split}
        \sum_{l'\in\mathbb{N}} \quad \sum_{l\in k+2l'N-B_N} e^{-l^2/(2\sigma^2)} = \sum_{l=k+N+1}^\infty e^{-l^2/(2\sigma^2)}
        &\leq\int_{N+k}^\infty e^{-x^2/(2\sigma^2)}\,\mathrm{d}x\\
        &\leq \frac{C\sigma^2e^{-(N+k)^2/(2\sigma^2)}}{N+k}\\
        &\leq \frac{C\sigma^2e^{-(N-\mathfrak{m})^2/(2\sigma^2)}}{N-\mathfrak{m}},
    \end{split}
\end{equation*}
where the Mills' ratio \eqref{Mill} has entered at the second to last inequality. An application of the embedding of the $\ell^p$ spaces then leads to
\begin{equation}\label{botint}
    \begin{split}
        \frac{1}{(N-\mathfrak{m})^\alpha}\bigg(\sum_{l'\in\mathbb{N}} \bigg| \sum_{l\in k+2l'N-B_N} e^{-l^2/(2\sigma^2)}\bigg|^2\bigg)^{1/2} &\leq \frac{1}{(N-\mathfrak{m})^\alpha} \sum_{l=k+N+1}^\infty e^{-l^2/(2\sigma^2)} \\
        &\leq \frac{C\sigma^2e^{-(N-\mathfrak{m})^2/(2\sigma^2)}}{(N-\mathfrak{m})^{\alpha+1}}.
    \end{split}    
\end{equation}
By combining \eqref{murint}, \eqref{Ijl}, \eqref{contint} \eqref{sideint}, \eqref{botint}, along with \eqref{G(r)}, we obtain 
\begin{equation*}
    \begin{split}
        \bigg(\sum_{l'\in\mathbb{N}} |\mu_{r,\sigma}(k+2l'N) &|^2\bigg)^{1/2} \\
        &\leq \frac{C\sigma^2e^{-(N-\mathfrak{m})^2/(2\sigma^2)}}{(N-\mathfrak{m})^{\alpha+1}} + C\mathcal{G}_\sigma(r) \\
        &\leq C(N-\mathfrak{m})^{2\beta-\alpha-1}e^{-3\pi^2(N-\mathfrak{m})^{2(1-\beta)}} + C(N-\mathfrak{m})^{\beta} e^{-3\pi^2(N-\mathfrak{m})^{2(\beta-\alpha)}}\\
        &= \widetilde{\mathcal{E}}_2(N,\mathfrak{m},\alpha,\beta),
    \end{split} 
\end{equation*}
from which and the observation made at the start of the proof, \eqref{mul2} follows. 
\end{proof}

\begin{proof}[Proof of Theorem~\ref{thm:TsamplingregwG}]
The proof is a combination Lemma~\ref{lem:regsamp} and its proof, \eqref{l2replace*}, and Propositions~\ref{prop:nu},~\ref{prop:mu}.
\end{proof}
 
\section{Proof of Theorem~\ref{thm:WNN}} \label{sec:WNNthm}

We begin by reminding the reader of some useful key points for this section. Let $f \in \mathcal{B}_\mathfrak{m}$, and let $N > \mathfrak{m}$. 
We have asserted that 
\begin{equation}\label{gpapprox}
    \mathcal{R}_{r,\sigma}f(x) = \sum_{j=0}^{2N-1} f\Big(\frac{j}{2N}\Big)(s_N\,\mathcal{G}_{r,\sigma})(x-j/2N) \approx f(x), 
\end{equation}
in $L^2([0,1];\mathbb{C}^m)$, with a precise error bound provided in Theorem~\ref{thm:TsamplingregwG}. Importantly, we specify in $\mathcal{G}_{r,\sigma}$,
\begin{equation*}
    r = \frac{3\pi}{(N-\mathfrak{m})^{\alpha}} \quad\text{ and }\quad \sigma = \frac{(N-\mathfrak{m})^{\beta}}{\sqrt{6}\pi},
\end{equation*}
where $0<\alpha<\beta<1$. To facilitate the ensuing calculations, we identify $s_N$, $\mathcal{G}_{r,\sigma}$ with their $1$-periodic counterparts. Then $s_N\in C^{\infty}(\mathbb{R})$, and from \eqref{def:Gtruncder}, $\mathcal{G}_{r,\sigma}^{(j)}$ exists almost everywhere for every $j\in\mathbb{N}$, where 
\begin{equation}\label{def:Grprime}
    \mathcal{G}_{r,\sigma}^{(j)}(x)=\begin{cases}
                                \mathcal{G}_{\sigma}^{(j)}(x),  &\exists k\in\mathbb{Z}\ \mathrm{s.t.}\ k - r \leq x \leq k + r, \\
                                0, &\mathrm{otherwise}.
                              \end{cases}
\end{equation}
Moreover, we have established in Subsection~\ref{sec:WNNframework} that the output of our WNN will take the form
\begin{equation}\label{WNNoutputrecall}
    \sum_{j=0}^{2N-1} f\Big(\frac{j}{2N}\Big)T_{\mathcal{K}_j}\rho_j(x), \quad\forall x\in [0,1].
\end{equation}
Placing \eqref{gpapprox} alongside \eqref{WNNoutputrecall}, it is clear that most of the proof of Theorem~\ref{thm:WNN} will be complete if we can demonstrate 
\begin{equation}\label{filterstep}
    T_{\mathcal{K}_j}\rho_j(x) = (s_N\,\mathcal{G}_{r,\sigma})(x - j/2N) + \mathcal{E},
\end{equation}
for a small $\mathcal{E}$ error term and $x$ in an appropriately large subset of $[0,1]$. If further, $\mathcal{E}$ vanishes as $N \to \infty$, then it entails that the overall discrepancy between the network output and $f$ can be minimized for the same subset.

We proceed to bound the error in \eqref{filterstep} with the following proposition.

\begin{proposition} \label{prop:WNNthmcont} Let $N, \mathfrak{m}\in\mathbb{N}$. Suppose $W$ satisfies Assumption~\ref{assump:regular}, and $N>\mathfrak{m}$ is sufficiently large so that 
\begin{equation}\label{rkappa}
    r = \frac{3\pi}{(N-\mathfrak{m})^{\alpha}} \leq \kappa
\end{equation}
where $\kappa$ is as in \eqref{def:diagregion}. Let $\beta-\alpha = 1 -\beta$. Fix $x\in [r,1-r]$ and $j=0,\cdots, 2N-1$. Suppose $x-j/2N\in [-r,r]$. There exists a universal constant $C>0$ such that
\begin{multline} \label{eq:WNNerror1}
    |T_{\mathcal{K}_j}\rho_j(x)-(s_N\,\mathcal{G}_{r,\sigma})(x-j/2N)| \\
    \leq C\eta^{-1}\bigg(KNr^2 + N(N-\mathfrak{m})^{\beta} e^{-3\pi^2(N-\mathfrak{m})^{2(\beta-\alpha)}/2}\bigg).
\end{multline}
\end{proposition}

The proof of Proposition~\ref{prop:WNNthmcont} relies on the following auxiliary lemma, which provides estimates of $s_N\,\mathcal{G}_{r,\sigma}$ and its derivatives. These estimates are crucial not only for this section but also for the subsequent ones. Its proof is given in Subsection~\ref{sec:Gsizelem}.

\begin{lemma} \label{lem:Gsize} 
Given a pair of $r,\sigma$ as in \eqref{rsigmapair}, with $0<\alpha<\beta<1$. There exists a universal constant $C> 0$ such that for all $N>\mathfrak{m}$ we have
\begin{enumerate}
    \item[(a)] $|(s_N\,\mathcal{G}_{r,\sigma})'(r)| \leq C\max\{(N-\mathfrak{m})^{3\beta-\alpha}, N(N-\mathfrak{m})^{\beta}\} \, e^{-3\pi^2(N-\mathfrak{m})^{2(\beta-\alpha)}/2}$, 
    \item[(b)] $\|(s_N\,\mathcal{G}_{r,\sigma})'\|_{L^{\infty}([0,1])} \leq C\max\{N, N^{2\beta-\alpha}\}$
    \item[(c)] $\|(s_N\,\mathcal{G}_{r,\sigma})''\|_{L^{\infty}([0,1])} \leq C\max\{N^2, N^{4\beta-2\alpha}, N^{2\beta-\alpha+1} \}$, 
    \item[(d)] $\|(s_N\,\mathcal{G}_{r,\sigma})'''\|_{L^{\infty}([0,1])} \leq C\max\{N^3, N^{6\beta-3\alpha}, N^{4\beta-2\alpha+1}, N^{2\beta-\alpha+2}\}$.
\end{enumerate}
\end{lemma}


\begin{proof}[Proof of Proposition~\ref{prop:WNNthmcont}]

Begin by noting that if $x\in [r,1-r]$, then from definition \eqref{Wxbar}, $\mathfrak{I}_x = [-r+x,x+r]$. 
Since $x-j/2N\in [-r,r]$, we can write
\begin{equation*}
    \int_{\mathfrak{I}_x} \rho_j(y) \, \frac{\mathrm{d}^2}{\mathrm{d} y^2}(s_N\,\mathcal{G}_{r,\sigma})(x-y)\,\mathrm{d}y = \int_{j/2N}^{x+r} (y-j/2N) \, \frac{\mathrm{d}^2}{\mathrm{d} y^2}(s_N\,\mathcal{G}_{r,\sigma})(x-y) \,\mathrm{d}y.
\end{equation*}
Therefore,
\begin{multline} \label{E1in}
    \int_{\mathfrak{I}_x} \rho_j(y) \, \frac{\mathrm{d}^2}{\mathrm{d} y^2}(s_N\,\mathcal{G}_{r,\sigma})(x-y) \,\mathrm{d}y \\
    =(s_N\,\mathcal{G}_{r,\sigma})(x-j/2N) - \Big( (x+r-j/2N)(s_N\,\mathcal{G}_{r,\sigma})'(-r) +  (s_N\,\mathcal{G}_{r,\sigma})(-r)\Big).
\end{multline}
On account of $1-\beta=\beta-\alpha$, we deduce from Lemma~\ref{lem:Gsize} and \eqref{G(r)}, \eqref{E1in} that
\begin{multline}\label{Ebound}
    \bigg|\int_{\mathfrak{I}_x} \rho_j(y) \, \frac{\mathrm{d}^2}{\mathrm{d} y^2}(s_N\,\mathcal{G}_{r,\sigma})(x-y)\,\mathrm{d}y - (s_N\,\mathcal{G}_{r,\sigma})(x-j/2N) \bigg| \\
    \leq CN(N-\mathfrak{m})^{\beta} e^{-3\pi^2(N-\mathfrak{m})^{2(\beta-\alpha)}/2}.
\end{multline}
Continuing, we compare 
$T_{\mathcal{K}_j}\rho_j(x)$ with $\int_{\mathfrak{I}_x} \rho_j(y)\frac{\mathrm{d}^2}{\mathrm{d} y^2}(s_N\,\mathcal{G}_{r,\sigma})(x-y)\,\mathrm{d}y$. Due to \eqref{rkappa}, $\{(x,y): y\in \mathfrak{I}_x\} \subset\mathcal{D}_\kappa$. Hence, it follows from Assumption~\ref{assump:regular} that $\mathcal{W}_{x}\geq\eta$ and 
\begin{equation}\label{Lip}
\begin{split}
    |\mathcal{W}_{x} - W(x,y)| &= \bigg|\frac{1}{|\mathfrak{I}_x|}\int_{\mathfrak{I}_x} (W(x,u)- W(x,y))\,\mathrm{d}u \bigg| \\ 
    &\leq \frac{1}{|\mathfrak{I}_x|}\int_{\mathfrak{I}_x} |W(x,u)-W(x,y)|\,\mathrm{d}u \\ 
    &\leq 2 Kr
\end{split}
\end{equation}
for almost every $y\in \mathfrak{I}_x$. Moreover, by Radamacher's theorem \citep[\S 5.8.3, Theorem 5]{evans1998partial} and the Lipschitz continuity assumption, 
$\frac{{\rm d}}{{\rm d}y} W(x,y)$ exists almost everywhere on $\mathfrak{I}_{x}$, and that
\begin{equation}\label{Kbd}
    \bigg|\frac{\mathrm{d}}{\mathrm{d}y}W(x,y)\bigg|\leq K.
\end{equation}
Thus
\begin{multline} \label{tighter}
    \int_{\mathfrak{I}_x} \rho_{j}(y)\, \frac{\mathrm{d}^2}{\mathrm{d} y^2}(s_N\,\mathcal{G}_{r,\sigma})(x-y)\bigg(W(x,y)-\mathcal{W}_{x}\bigg) \,\mathrm{d}y \\
    = \int_{j/2N}^{x+r}(y-j/2N)\, \frac{\mathrm{d}^2}{\mathrm{d} y^2}(s_N\,\mathcal{G}_{r,\sigma})(x-y)\bigg(W(x,y)-\mathcal{W}_{x}\bigg)\,\mathrm{d}y 
    = \mathcal{E}_1 + \mathcal{E}_2,
\end{multline}
where
\begin{align*}
    \mathcal{E}_1 &:= -\int_{j/2N}^{x+r} \frac{{\rm d}}{{\rm d}y} (s_N\,\mathcal{G}_{r,\sigma})(x-y)\bigg((y-j/2N) \,\frac{{\rm d}}{{\rm d}y} W(x,y) + W(x,y)-\mathcal{W}_{x}\bigg)\,\mathrm{d}y, \\
    \mathcal{E}_2 &:= (y-j/2N)\,\frac{{\rm d}}{{\rm d}y}(s_N\,\mathcal{G}_{r,\sigma})(x-y) \bigg(W(x,y)-\mathcal{W}_{x}\bigg)\bigg|_{y=j/2N}^{y=x+r}.
\end{align*}
As a consequence of \eqref{Lip}, \eqref{Kbd}, and Lemma~\ref{lem:Gsize},
\begin{equation}\label{E2bd}
    \begin{split}
        |\mathcal{E}_1| &\leq CN \bigg(K\int_0^{2r} y\,\mathrm{d}y + Kr^2 \bigg) \leq CKN r^2, \\
        |\mathcal{E}_2| &\leq C N(N-\mathfrak{m})^{\beta} e^{-3\pi^2(N-\mathfrak{m})^{2(\beta-\alpha)}/2}.
    \end{split}  
\end{equation}
Then, by synthesizing the findings of \eqref{tighter}, \eqref{E2bd}, we acquire, for $x-j/2N\in [-r,r]$,
\begin{equation} \label{oneend}
    \begin{split}
        \bigg|T_{\mathcal{K}_j}\rho_{j}(x) - \int_{\mathfrak{I}_x} \rho_j(y) \, &\frac{\mathrm{d}^2}{\mathrm{d} y^2}(s_N\,\mathcal{G}_{r,\sigma})(x-y)\,\mathrm{d}y \bigg| \\
        &= \frac{1}{\mathcal{W}_x}\bigg|\int_{\mathfrak{I}_x} \rho_{j}(y)\, \frac{\mathrm{d}^2}{\mathrm{d} y^2}(s_N\,\mathcal{G}_{r,\sigma})(x-y)\bigg(W(x,y)-\mathcal{W}_x\bigg)\,\mathrm{d}y\bigg|\\
        &\leq C\eta^{-1}\bigg(KN r^2 + N(N-\mathfrak{m})^{\beta} e^{-3\pi^2(N-\mathfrak{m})^{2(\beta-\alpha)}/2}\bigg).
    \end{split}
\end{equation}
Combining \eqref{Ebound}, \eqref{oneend}, we arrive at \eqref{eq:WNNerror1}, proving the proposition. 
\end{proof} 

Having established the proof of Proposition~\ref{prop:WNNthmcont}, we move to the proof of Theorem~\ref{thm:WNN}. 

\begin{proof}[Proof of Theorem~\ref{thm:WNN}]
Since $\|f\|_{L^2([0,1];\mathbb{C}^m)} = 1$, it means, from Lemma~\ref{lem:tech},
\begin{equation}\label{norm1}
    \sum_{j=0}^{2N-1} |f\Big(\frac{j}{2N}\Big)|^2 =1.
\end{equation}
We assume that \eqref{rkappa} holds and that $1-\beta=\beta-\alpha$. 
Let $x\in [r,1-r]$.
Observe, if $|x-j/2N|>r$, then either $s_N\,\mathcal{G}_{r,\sigma}(x-j/2N)=0$ or
\begin{equation*}
    T_{\mathcal{K}_j}\rho_j(x) = \frac{\chi_{\{x-j/2N\leq r\}}(x)}{\mathcal{W}_x} \int_{\mathfrak{I}_x} \rho_j(y)\, \frac{\mathrm{d}^2}{\mathrm{d} y^2}(s_N\,\mathcal{G}_{r,\sigma})(x-y) W(x,y) \, \mathrm{d}y = 0.
\end{equation*}
Therefore, it suffices to consider only indices $j=0,\cdots, 2N-1$ such that $|x-j/2N|\leq r$.
Note that there are only at most $\lceil rN\rceil$ such indices. 
Next, we recall the WNN construction given in Subsection~\ref{sec:WNNframework}, 
\begin{equation*}
    \Psi_f(x) = \sum_{j=0}^{2N-1} f\Big(\frac{j}{2N}\Big)T_{\mathcal{K}_j} \rho_j(x),
\end{equation*}
as well as the regularized sampling series
\begin{equation*} 
    \mathcal{R}_{r,\sigma}f(x) = \sum_{j=0}^{2N-1} f\Big(\frac{j}{2N}\Big)(s_N\,\mathcal{G}_{r,\sigma})(x-j/2N).
\end{equation*}
Then as a direct consequence of \eqref{norm1}, Proposition~\ref{prop:WNNthmcont} and the $L^p$-embedding property for finite measure spaces, 
\begin{equation} \label{eq:regseriesWNN}
    \begin{split}
        |\Psi_f(x)-\mathcal{R}_{r,\sigma}f(x)| &\leq \sum_{j=0}^{2N-1} |f\Big(\frac{j}{2N}\Big)|\, |T_{\mathcal{K}_j} \rho_j(x) - (s_N\,\mathcal{G}_{r,\sigma})(x-j/2N)|\\
        &\leq C\eta^{-1} r^{1/2}N^{1/2} \bigg(KNr^2 + N(N-\mathfrak{m})^{\beta} e^{-3\pi^2(N-\mathfrak{m})^{2(\beta-\alpha)}/2}\bigg).
    \end{split}
\end{equation}
Combining \eqref{eq:regseriesWNN} with the result of Theorem~\ref{thm:TsamplingregwG}, we obtain for all $N\gg \mathfrak{m}$,
\begin{equation}\label{laststep}
    \|\Psi_f-f\|_{L^2([r,1-r];\mathbb{C}^m)}\leq C\eta^{-1} K N^{3/2-\alpha 5/2}.
\end{equation}
Let
\begin{equation}\label{sys}
    3/2 - \alpha 5/2 = 0.9. 
\end{equation}
Then $\alpha=0.96$, $\beta=0.98$. Further, select
\begin{equation} \label{Nchoice}
    N=\lceil \varepsilon^{-10/9}\rceil.
\end{equation}
Then it follows from \eqref{laststep}, \eqref{sys} that
\begin{equation*}
    \|\Psi_f-f\|_{L^2([r,1-r];\mathbb{C}^m)} \leq C\eta^{-1}K\varepsilon.
\end{equation*}
Finally, we note from \eqref{Nchoice} that \eqref{rkappa} can be guaranteed by taking $\varepsilon\in (0,1)$ to be sufficiently small.
\end{proof}

\begin{remark}\label{rem:flat}
It is important to note that the favorable subgeometric sampling convergence rate established Theorem~\ref{thm:TsamplingregwG} has been forfeited in the proof of Theorem~\ref{thm:WNN} due to the implementation of the integration by parts demonstrated in the proof of Proposition~\ref{prop:WNNthmcont}. 
Let us suppose for a moment that $W$ is constant over the diagonal neighborhood $\mathcal{D}_\kappa$, in which case, a diligent study of the proof presented would lead to the conclusion that 
\begin{equation*}
    \|\Psi_f-f\|_{L^2([r,1-r];\mathbb{C}^m)}=\mathcal{O}(e^{-CN^{2(\beta-\alpha)}}),
\end{equation*}
and it would have been enough to select $N=C\lceil (\log(1/\varepsilon))^{1/(2(\beta-\alpha))}\rceil$. 
\end{remark}

\subsection{Proof of Lemma~\ref{lem:Gsize}} \label{sec:Gsizelem}

For $k=1,2,3$ we have from \eqref{def:Grprime} that
\begin{equation}\label{ders}
    (s_N\,\mathcal{G}_{r,\sigma})^{(k)}(x) = \sum_{l=0}^k {k \choose l} (s^{(l)}\,\mathcal{G}_\sigma^{(k-l)})(x),\quad \forall x\in [-r,r].
\end{equation}
We first prove (a), so we take $k=1$. Utilizing the Poisson formulation \eqref{Poisson_app}
\begin{equation*} 
    \mathcal{G}_{\sigma}(x)= \sum_{l\in\mathbb{Z}} \mathfrak{g}(x+l), \quad \forall x\in [-1/2,1/2).
\end{equation*}
where $\mathfrak{g}(x) = d(\sigma)e^{-2x^2\pi^2\sigma^2}$ on $\mathbb{R}$ and $d(\sigma) = \sigma c(\sigma)\sqrt{2\pi}$, we deduce
\begin{align} \label{eq:Gprimesplit}
    \nonumber \mathcal{G}_\sigma'(r) = \sum_{l\in\mathbb{Z}} \mathfrak{g}'(r+l) &= -4\pi^2\sigma^2 d(\sigma)\sum_{l\in\mathbb{Z}} (r+l)e^{-2(r+l)^2\pi^2\sigma^2}\\
    &= -4\pi^2\sigma^2 r\mathcal{G}_{\sigma}(r) -4\pi^2\sigma^2 d(\sigma)\sum_{l\in\mathbb{Z}} le^{-2(r+l)^2\pi^2\sigma^2}.
\end{align}
To handle the first term on the right-hand side of \eqref{eq:Gprimesplit}, we recall from \eqref{G(r)} that
\begin{equation} \label{Grsize}
    |\mathcal{G}_\sigma(r)|\leq C(N-\mathfrak{m})^{\beta}e^{-3\pi^2(N-\mathfrak{m})^{2(\beta-\alpha)}}.
\end{equation}
Thus,
\begin{equation} \label{eq:Gprimer1}
    |\sigma^2 r\mathcal{G}_{\sigma}(r)| \leq C(N-\mathfrak{m})^{3\beta-\alpha}e^{-3\pi^2(N-\mathfrak{m})^{2(\beta-\alpha)}}.
\end{equation}
For the second term, by taking note that $r\in (0,1/2)$ and applying Mills' ratio \eqref{Mills} along with simple calculations, we obtain 
\begin{equation}\label{eq:Gprimer2}
    \begin{split}
        \Big|\sigma^2 d(\sigma)\sum_{l\in\mathbb{Z}} le^{-2(r+l)^2\pi^2\sigma^2} \Big| 
        &\leq C\sigma^2 d(\sigma) \sum_{l\in\mathbb{N}} (l-r) e^{-2\pi^2(l-r)^2\sigma^{2}} \\ 
        &\leq C\sigma^2 \sum_{l\in\mathbb{N}} e^{-\pi^2(l-r)^2\sigma^{2}} \\ 
        &\leq C\sigma^2 \bigg( e^{-\pi^2(1-r)^2\sigma^2} + \int_{1-r}^{\infty} e^{-\pi^2y^2\sigma^{2}} \,\mathrm{d}y \bigg) \\
        &\leq C\sigma^2 e^{-\pi^2\sigma^2r^2} \\ 
        &\leq C(N-\mathfrak{m})^{2\beta} e^{-3\pi^2(N-\mathfrak{m})^{2(\beta-\alpha)}/2}. 
    \end{split}
\end{equation}
Together, \eqref{eq:Gprimesplit}, \eqref{eq:Gprimer1}, \eqref{eq:Gprimer2} deliver
\begin{equation} \label{Gprimer}
    |\mathcal{G}'_{\sigma}(r)| \leq C(N-\mathfrak{m})^{3\beta-\alpha}e^{-3\pi^2(N-\mathfrak{m})^{2(\beta-\alpha)}/2}.
\end{equation}
Moreover, it is evident that
\begin{equation}\label{sampfsize}
    \|s_N'\|_{L^{\infty}([0,1])}\leq CN, \quad \|s_N''\|_{L^{\infty}([0,1])}\leq CN^2, \quad \|s_N'''\|_{L^{\infty},([0,1])}\leq CN^3.
\end{equation}
Thus, combining \eqref{ders}, \eqref{Grsize}, \eqref{Gprimer}, \eqref{sampfsize}, with that $|\mathcal{G}_\sigma(r)|=|\mathcal{G}_{r,\sigma}(r)|$, $|\mathcal{G}'_\sigma(r)|=|\mathcal{G}'_{r,\sigma}(r)|$, we have the proof of statement (a).

We will now derive statements (b), (c), (d) simultaneously. Let $x\in [0,r]$. Following the argument presented in \eqref{eq:Gprimesplit}, it can be concluded that
\begin{equation*} 
    \mathcal{G}_\sigma'(x) = -4\pi^2\sigma^2 x\mathcal{G}_{\sigma}(x) -4\pi^2\sigma^2 d(\sigma)\sum_{l\in\mathbb{Z}} le^{-2(x+l)^2\pi^2\sigma^2},
\end{equation*}
and so
\begin{equation*} 
    \mathcal{G}_\sigma''(x) = -4\pi^2\sigma^2 \mathcal{G}_{\sigma}(x) - 4\pi^2\sigma^2 x\mathcal{G}_{\sigma}'(x) + 16\pi^4\sigma^4 d(\sigma)\sum_{l\in\mathbb{Z}} l(x+l)e^{-2(x+l)^2\pi^2\sigma^2},
\end{equation*}
and
\begin{multline*} 
    \mathcal{G}_\sigma'''(x) = -8\pi^2\sigma^2 \mathcal{G}_{\sigma}'(x) - 4\pi^2\sigma^2 x\mathcal{G}_{\sigma}''(x) + 16\pi^4\sigma^4 d(\sigma)\sum_{l\in\mathbb{Z}} le^{-2(x+l)^2\pi^2\sigma^2}
    \\
    - 64\pi^6\sigma^6 d(\sigma)\sum_{l\in\mathbb{Z}} l(x+l)^2 e^{-2(x+l)^2\pi^2\sigma^2}.
\end{multline*}
Then by the same manipulations used in \eqref{eq:Gprimer1}, \eqref{eq:Gprimer2}, we deduce
\begin{alignat}{2}    
    \label{eq:firstdermax}
    &|\mathcal{G}_{\sigma}'(x)| &&\leq C(N-\mathfrak{m})^{2\beta-\alpha}, \\
    \label{eq:seconddermax}
    &|\mathcal{G}_{\sigma}''(x)| &&\leq C(N-\mathfrak{m})^{4\beta-2\alpha}, \\
    \label{eq:thirddermax}
    &|\mathcal{G}_{\sigma}'''(x)| &&\leq C(N-\mathfrak{m})^{6\beta-3\alpha}.
\end{alignat}
Integrating \eqref{eq:firstdermax}, \eqref{eq:seconddermax}, \eqref{eq:thirddermax} with \eqref{sampfsize}, while considering \eqref{def:Grprime} and symmetry, we confirm (b), (c), (d), as desired.
\qed

\section{Proof of Theorem~\ref{thm:GNNdet}} \label{sec:GNNdet}

We start with some preparation. We continue to assume \eqref{rkappa} holds with $\alpha=0.96$. Let the graph $G^{\mathrm{det}}_n$ and its graph signal $f_n$ be as in the premise of Theorem~\ref{thm:GNNdet}. As illustrated in Subsections~\ref{sec:WNN},~\ref{sec:WNNframework}, we can identify the graph signal $f_n$ with a function on $\mathcal{X}_n = \{x_1,\cdots,x_n\}$, where
\begin{equation*}
    x_k = \frac{k-1}{n}\in [0,1].
\end{equation*}
Recall that $f_n(x_k) = f(x_k)$ for every $x_k$, and that the GNN instantiated by $\Psi_{n,f}$ is expressed as follows,
\begin{equation} \label{eq:GNNrecall}
    \Psi_{n,f}(x_k) = \sum_{j=0}^{2N-1} f\Big(\frac{j}{2N}\Big)\mathfrak{F}_j(\rho_j)(x_k).
\end{equation}
Since $f \in \mathcal{B}_\mathfrak{m}$, Theorem~\ref{thm:WNN} implies that the WNN machinery described therein produces the outcome
\begin{equation*}
    \Psi_f(x) = \sum_{j=0}^{2N-1} f\Big(\frac{j}{2N}\Big)T_{\mathcal{K}_j} \rho_j(x) \approx f(x)
\end{equation*}
in $L^2$-norm on the predictable zone $[r,1-r]$. Therefore, establishing that
\begin{equation}\label{handwaving}
    \mathfrak{F}_j(\rho_j)(x)\approx T_{\mathcal{K}_j}(\rho_j)(x)
\end{equation}
in some appropriate sense, would complete a substantial portion of our proof. However, as $\mathfrak{F}_j(\rho_j)$ is a graph signal and $T_{\mathcal{K}_j}(\rho_j)$ a graphon signal, we need to extend the former to a graphon signal to make sense of \eqref{handwaving}. We do this next.

We introduce an abbreviation to be used throughout the remainder of this paper, 
to ensure a more concise presentation:
\begin{equation*}
    \mathcal{G}^*(x,y):=\frac{\mathrm{d}^2}{\mathrm{d} y^2}(s_N\,\mathcal{G}_{r,\sigma})(x-y).
\end{equation*}
For $j=0,\cdots, 2N-1$, the graph filter $\mathfrak{F}_j$ acts on a graph signal $g_n$ as follows (see \eqref{frakF}, \eqref{graphker}) 
\begin{equation*}
    \begin{split}
        \mathfrak{F}_j(g_n)(x_k) 
        &= \frac{\chi_{\{x_k - j/2N\leq r\}}(x_k)}{n\mathcal{W}_{x_k}} \sum_{l=1}^n \mathcal{G}^*(x_k,x_l)[{\bf A}_n]_{k,l}g_n(x_l)\\
        &= \frac{\chi_{\{x_k - j/2N\leq r\}}(x_k)}{n\mathcal{W}_{x_k}} 
        \sum_{l\not=k, l=1}^n \mathcal{G}^*(x_k,x_l)W(x_k,x_l)g_n(x_l).
    \end{split}
\end{equation*}
Define a step graphon signal
\begin{equation}\label{def:extendFrho}
    \overline{\mathfrak{F}_j(\rho_j)}(x) := \sum_{k=1}^n \mathfrak{F}_j(\rho_j)(x_k)\chi_{I_k}(x) \quad\forall x\in [0,1],
\end{equation}
where, recall that $I_k= [(k-1)/n, k/n) = [x_k, x_{k+1})$. Note that, $\overline{\mathfrak{F}_j(\rho_j)}(x_k)=\mathfrak{F}_j(\rho_j)(x_k)$ for $k=1,\cdots,n$. Moreover, by a combination of \eqref{eq:GNNrecall}, \eqref{def:extendFrho} and definition \eqref{def:barPsinf},
\begin{equation}\label{GNNextend}
    \begin{split}
        \overline{\Psi}_{n,f}(x) &= \sum_{k=1}^{n} \Psi_{n,f}(x_k)\chi_{I_k}(x) = \sum_{k=1}^{n}\bigg(\sum_{j=0}^{2N-1}f\Big(\frac{j}{2N}\Big)\mathfrak{F}_j(\rho_j)(x_k)\bigg)\chi_{I_k}(x)\\
        &= \sum_{j=0}^{2N-1}f\Big(\frac{j}{2N}\Big)\overline{\mathfrak{F}_j(\rho_j)}(x).
    \end{split}
\end{equation}
With this, we can now formalize \eqref{handwaving}. 

\begin{lemma} \label{lem:kerneldiff} Let $\alpha, \beta$ be as prescribed by Theorem~\ref{thm:GNNdet}. There exists a universal constant $C>0$ such that, for each $k=1,\cdots,n$ and $j=0,\cdots, 2N-1$, we have
\begin{equation}\label{kerneldifferr}
    |T_{\mathcal{K}_j}\rho_j(x_k) - \overline{\mathfrak{F}_j(\rho_j)}(x_k)|\leq C\eta^{-1}(1+K)\bigg(\frac{N^3}{n^2} + \frac{N^{3-\alpha}}{n}\bigg).
\end{equation}
\end{lemma}

The proof of Lemma~\ref{lem:kerneldiff} is given in Subsection~\ref{sec:kerneldiff}, at the end of this section, so as for us to proceed with the proof of Theorem~\ref{thm:GNNdet}.

\begin{proof}[Proof of Theorem~\ref{thm:GNNdet}] Without loss of generality, we assume $\|f\|_{L^2([0,1];\mathbb{C}^m)}=1$ as usual. 
Let us first examine the difference
\begin{equation*}
    \bigg\| \sum_{j=0}^{2N-1} f\Big(\frac{j}{2N}\Big) \sum_{k=1}^n T_{\mathcal{K}_j}  \rho_j(x_k) \chi_{I_k} - \sum_{j=0}^{2N-1} f\Big(\frac{j}{2N}\Big) T_{\mathcal{K}_j} \rho_j \bigg\|_{L^1([r,1-r];\mathbb{C}^m)}.
\end{equation*}
It can be readily assumed that
\begin{equation}\label{1stthresholdforn}
    n>2N >N^{\alpha}.
\end{equation}
This condition implies that, for each $k=1,\cdots, n$, and for every $x\in I_k = [x_k, x_{k+1})$, there exist at most $\lceil rN\rceil + 1$ locations $j/2N\in \mathfrak{I}_{x_k}\cup\mathfrak{I}_x$. 
Note that $T_{\mathcal{K}_j}\rho_j(x_k) = 0 = T_{\mathcal{K}_j}\rho_j(x)$ if $j/2N\not\in \mathfrak{I}_{x_k}\cup\mathfrak{I}_x$.
Moreover, $I_k\subset\mathfrak{I}_{x_k}$. 
Fix $k=1,\cdots, n$, and let $x\in I_k$. 
Consider $j/2N\in \mathfrak{I}_{x_k}\cap \mathfrak{I}_x$. We attain that
\begin{align}
    \nonumber |T_{\mathcal{K}_j}\rho_j(x_k) - T_{\mathcal{K}_j}\rho_j(x)| &= \bigg| \int_0^1 \bigg(\frac{\mathcal{G}^*(x_k,y)W(x_k,y)}{\mathcal{W}_{x_k}} - \frac{\mathcal{G}^*(x,y)W(x,y)}{\mathcal{W}_{x}}\bigg)\rho_j(y)\,\mathrm{d}y\bigg|\\
    \nonumber &\leq \bigg| \int_{\mathfrak{I}_{x_k}\cap\mathfrak{I}_x} \bigg(\frac{\mathcal{G}^*(x_k,y)W(x_k,y)}{\mathcal{W}_{x_k}} - \frac{\mathcal{G}^*(x,y)W(x,y)}{\mathcal{W}_{x}}\bigg)\rho_j(y)\,\mathrm{d}y\bigg|\\
    \label{bulk} &\leq \frac{C\eta^{-2} KN^{2-2\alpha}}{n} + \frac{C\eta^{-1} N^{3-2\alpha}}{n} + \frac{C\eta^{-1}KN^{2-2\alpha}}{n},
\end{align}
due to the following. First, by \eqref{1stthresholdforn}, $|\mathfrak{I}_{x_k}\cap \mathfrak{I}_x|\leq 2r\leq CN^{-\alpha}$; thus
\begin{equation*}
    \int_{\mathfrak{I}_{x_k}\cap\mathfrak{I}_x} \rho_j(y)\,\mathrm{d}y \leq \int_0^{2r} y\, \mathrm{d}y \leq CN^{-2\alpha}.
\end{equation*}
Second, by Assumption~\ref{assump:regular}, $|W(x_k,y)-W(x,y)|\leq K/n$ for every $y\in \mathfrak{I}_{x_k}\cap\mathfrak{I}_x$. Third, by Lemma~\ref{lem:Gsize}, $|\mathcal{G}^*(x,y)|, |\mathcal{G}^*(x_k,y)|\leq CN^2$, and
\begin{equation*}
    |\mathcal{G}^*(x_k,y)-\mathcal{G}^*(x,y)|\leq \frac{CN^3}{n}.
\end{equation*}
Lastly, similar to the calculation in \eqref{Lip}
\begin{equation*}
    \bigg| \frac{1}{\mathcal{W}_{x_k}} - \frac{1}{\mathcal{W}_{x}} \bigg| \leq \eta^{-2}\bigg(|\mathcal{W}_{x_k}-W(x_k,x)|+|W(x_k,x)-\mathcal{W}_x|\bigg)\leq \frac{CK\eta^{-2}}{n}.
\end{equation*}
Next, we consider $j/2N$ in $(\mathfrak{I}_{x}\setminus\mathfrak{I}_{x_k})\cup (\mathfrak{I}_{x_k}\setminus\mathfrak{I}_{x})$. 
Note that, by \eqref{1stthresholdforn}, there can only be at most one $j/2N$ in $\mathfrak{I}_{x}\setminus\mathfrak{I}_{x_k}$ and one $j/2N$ in $\mathfrak{I}_{x_k}\setminus\mathfrak{I}_{x}\subset \mathfrak{I}_{x_k}\setminus\mathfrak{I}_{x_{k+1}}$.
In the former case, we obtain from Lemma~\ref{lem:Gsize}
\begin{equation} \label{end}
    \begin{split}
        |T_{\mathcal{K}_j} \rho_j(x_k) - T_{\mathcal{K}_j} \rho_j(x)| 
        =  
        |T_{\mathcal{K}_j} \rho_j(x)|
        &= \bigg| \int_{\mathfrak{I}_{x}\setminus\mathfrak{I}_{x_k}} \bigg(\frac{\mathcal{G}^*(x,y)W(x,y)}{\mathcal{W}_{x}}\bigg)\rho_j(y)\,\mathrm{d}y\bigg| \\
        &\leq C\eta^{-1} N^2 \int_0^{1/n} y\,\mathrm{d}y\leq \frac{C\eta^{-1}N^2}{n^2}.
    \end{split}
\end{equation}
However, in the latter case, we have
\begin{equation} \label{troubleend}
    \begin{split}
        |T_{\mathcal{K}_j} \rho_j(x_k) - T_{\mathcal{K}_j} \rho_j(x)| 
        =  
        |T_{\mathcal{K}_j} \rho_j(x_k)|
        &= \bigg| \int_{\mathfrak{I}_{x_k}} \bigg(\frac{\mathcal{G}^*(x_k,y)W(x_k,y)}{\mathcal{W}_{x_k}}\bigg)\rho_j(y)\,\mathrm{d}y\bigg| \\
        &\leq C\eta^{-1} N^2 \int_0^{2r} y\,\mathrm{d}y\leq C\eta^{-1}N^{2-2\alpha}.
    \end{split}
\end{equation}
Nevertheless, there are at most $2N$ intervals $I_{k}$ for which there exists one $j/2N \in \mathfrak{I}_{x_k}\setminus\mathfrak{I}_{x_{k+1}}$. 
Let $\mathcal{S}_1$ denote the union of such intervals.
We gather from \eqref{bulk}, \eqref{end}, \eqref{troubleend} and the normalization of $f$ that
\begin{align} \label{troubleset}
    \nonumber \bigg\|\sum_{j=0}^{2N-1} f\Big(\frac{j}{2N}\Big) \sum_{k=1}^n T_{\mathcal{K}_j} & \rho_j(x_k) \chi_{I_k} - \sum_{j=0}^{2N-1} f\Big(\frac{j}{2N}\Big) T_{\mathcal{K}_j} \rho_j \bigg\|_{L^1(\mathcal{S}_1;\mathbb{C}^m)} \\
    \nonumber &\leq  \bigg\|\sum_{k=1}^n \big( T_{\mathcal{K}_j} \rho_j(x_k)\chi_{I_k} - T_{\mathcal{K}_j} \rho_j\chi_{I_k}\big) \sum_{j=0}^{2N-1} f\Big(\frac{j}{2N}\Big) \bigg\|_{L^1(\mathcal{S}_1;\mathbb{C}^m)} \\
    &\leq C\eta^{-2}(1+K)\,\frac{N^{7/2-5\alpha/2}}{n}.  
\end{align}
Let $\mathcal{S}_2:= [r,1-r]\setminus\mathcal{S}_1$. 
Let $I_k=[x_k, x_{k+1})$ be such that $(x_k, x_{k+1})\cap \mathcal{S}_2\not=\emptyset$.
Then we deduce from construction, \eqref{1stthresholdforn}, \eqref{bulk}, \eqref{end}, that for every $x\in I_k = [x_k, x_{k+1})$, 
\begin{equation} \label{niceset}
    \begin{split}
        \sum_{j=0}^{2N-1}|f\Big(\frac{j}{2N}\Big)|\, |T_{\mathcal{K}_j}\rho_j(x)-T_{\mathcal{K}_j}\rho_j(x_k)|
        &\leq C\eta^{-2} (1+K) \sum_{j=0}^{2N-1}|f\Big(\frac{j}{2N}\Big)| \, \frac{N^{3-2\alpha}}{n} \\
        &\leq C\eta^{-2} (1+K) \, \frac{N^{7/2-5\alpha/2}}{n}.
    \end{split}
\end{equation}
Therefore, integrating \eqref{troubleset}, \eqref{niceset}, we obtain
\begin{multline} \label{variationdiff}
    \bigg\|\sum_{j=0}^{2N-1} f\Big(\frac{j}{2N}\Big) \sum_{k=1}^n T_{\mathcal{K}_j} \rho_j(x_k) \chi_{I_k} - \sum_{j=0}^{2N-1} f\Big(\frac{j}{2N}\Big) T_{\mathcal{K}_j} \rho_j \bigg\|_{L^1([r,1-r];\mathbb{C}^m)} \\
    \leq C\eta^{-2} (1+K) \, \frac{N^{7/2-5\alpha/2}}{n}.
\end{multline}
Moving on, we examine the difference
\begin{equation*}
    \bigg\| \sum_{j=0}^{2N-1} f\Big(\frac{j}{2N}\Big) \sum_{k=1}^n T_{\mathcal{K}_j} \rho_j(x_k) \chi_{I_k} - \sum_{j=0}^{2N-1}f\Big(\frac{j}{2N}\Big)\overline{\mathfrak{F}_j(\rho_j)} \bigg\|_{L^1([r,1-r];\mathbb{C}^m)}.
\end{equation*}
For each $k=1,\cdots, n$, it follows from the definition that $T_{\mathcal{K}_j} \rho_j(x_k)=0=\mathfrak{F}_j(\rho_j)(x_k)$ if $|x_k-j/2N|>r$. Hence, there are only at most $\lceil rN\rceil$ indices $j$ partaking in the following sum,
\begin{equation} \label{filtererror}
    \sum_{j=0}^{2N-1}|f\Big(\frac{j}{2N}\Big)|\, |T_{\mathcal{K}_j} \rho_j(x_k) - \overline{\mathfrak{F}_j(\rho_j)}(x_k)| 
    \leq C\eta^{-1}(1+K)N^{1/2-\alpha/2}\,\frac{N^{3 -\alpha}}{n}, 
\end{equation}
where the inequality follows from \eqref{1stthresholdforn} and Lemma~\ref{lem:kerneldiff}.
In turn, \eqref{filtererror} implies 
\begin{multline} \label{constantdiff}
    \bigg\| \sum_{j=0}^{2N-1} f\Big(\frac{j}{2N}\Big) \sum_{k=1}^n T_{\mathcal{K}_j} \rho_j(x_k) \chi_{I_k} - \sum_{j=0}^{2N-1}f\Big(\frac{j}{2N}\Big)\overline{\mathfrak{F}_j(\rho_j)} \bigg\|_{L^1([r,1-r];\mathbb{C}^m)} \\
    \leq C\eta^{-1}(1+K)\,\frac{N^{7/2 - 3\alpha/2}}{n}.
\end{multline}
Together, \eqref{variationdiff}, \eqref{constantdiff} deliver
\begin{equation*}
    \bigg\| \sum_{j=0}^{2N-1} f\Big(\frac{j}{2N}\Big) T_{\mathcal{K}_j} \rho_j - \sum_{j=0}^{2N-1}f\Big(\frac{j}{2N}\Big)\overline{\mathfrak{F}_j(\rho_j)} \bigg\|_{L^1([r,1-r];\mathbb{C}^m)} 
    \leq C\eta^{-2}(1+K)\,\frac{N^{7/2 - 3\alpha/2}}{n}.
\end{equation*}
Recalling that $\Psi_f = \sum_{j=0}^{2N-1} f\Big(\frac{j}{2N}\Big)T_{\mathcal{K}_j} \rho_j$, and from \eqref{GNNextend}, that $\overline{\Psi}_{n,f} = \sum_{j=0}^{2N-1}f\Big(\frac{j}{2N}\Big)\overline{\mathfrak{F}_j(\rho_j)}$, we conclude
\begin{equation*}
    \| \Psi_f - \overline{\Psi}_{n,f} \|_{L^1([r,1-r];\mathbb{C}^m)}
    \leq C\eta^{-2}(1+K)\,\frac{N^{7/2 - 3\alpha/2}}{n}.
\end{equation*}
An application of Theorem~\ref{thm:WNN} then gives
\begin{align} \label{nearfinGNNdet2}
    \nonumber \| \overline{\Psi}_{n,f} - f \|_{L^1([r,1-r];\mathbb{C}^m)} &\leq \| \overline{\Psi}_{n,f} - \Psi_f \|_{L^1([r,1-r];\mathbb{C}^m)} + \| \Psi_f - f \|_{L^1([r,1-r];\mathbb{C}^m)} \\
    \nonumber & \leq \| \overline{\Psi}_{n,f} - \Psi_f \|_{L^1([r,1-r];\mathbb{C}^m)} + \| \Psi_f - f \|_{L^2([r,1-r];\mathbb{C}^m)}\\
    &\leq C\eta^{-2}(1+K)\,\frac{N^{7/2 - 3\alpha/2}}{n} + C\eta^{-1}K\varepsilon.
\end{align}
By imposing $n\geq \frac{N^{7/2-3\alpha/2}}{\varepsilon}$, we can take $n=\lceil \varepsilon^{-10/3}\rceil = \lceil N^3 \rceil$. 
Putting this back in \eqref{nearfinGNNdet2} yields
\begin{equation*}
    \|\overline{\Psi}_{n,f}-f\|_{L^1([r,1-r];\mathbb{C}^m)} \leq C\eta^{-2}(1+K)\varepsilon.
\end{equation*}
The proof is now completed.
\end{proof} 

\subsection{Proof of Lemma~\ref{lem:kerneldiff}} \label{sec:kerneldiff}

For each $j=0,\cdots, 2N-1$, define
\begin{equation*}
    \overline{\rho}_j(x) :=\sum_{k=1}^n \rho_j(x_k)\chi_{I_k}(x), \quad\forall x\in [0,1].
\end{equation*}
We write
\begin{equation}\label{splitdiff}
    T_{\mathcal{K}_j} \rho_j - \overline{\mathfrak{F}_j(\rho_j)} = T_{\mathcal{K}_j} \rho_j - T_{\mathcal{K}_j} \overline{\rho}_j + T_{\mathcal{K}_j} \overline{\rho}_j - \overline{\mathfrak{F}_j (\rho_j)}.
\end{equation}
The difference $T_{\mathcal{K}_j} \rho_j - T_{\mathcal{K}_j} \overline{\rho}_j$ is straightforward to address. In fact, from Lemma~\ref{lem:Gsize} we obtain
\begin{equation}\label{1stdiff}
    \|T_{\mathcal{K}_j} (\rho_j-\overline{\rho}_j)\|_{L^{\infty}([0,1])} 
    \leq \eta^{-1}\|\mathcal{G}^*\|_{L^{\infty}([0,1])}\|\rho_j-\overline{\rho}_j\|_{L^1([0,1])} 
    \leq \frac{C\eta^{-1} N^2}{n}. 
\end{equation} 
The second to last inequality above is due to
\begin{equation*}
    \|\rho_j-\overline{\rho}_j\|_{L^1([0,1])} \leq \sum_{k=1}^n \frac{1}{n^2} = \frac{1}{n},
\end{equation*}
as 
the maximum difference $|\rho_j(x) - \rho_j(x_k)| \leq 1/n$ for every $x\in I_k = [(k-1)/n, k/n)$.

The second difference $T_{\mathcal{K}_j}\overline{\rho}_j - \overline{\mathfrak{F}(\rho_j)}$ on the right-hand side of \eqref{splitdiff} requires a more delicate treatment. Fix $k=1,\cdots, n$, and let $j=0,\cdots, 2N-1$. Since $|x_k-j/2N|>r$ implies $T_{\mathcal{K}_j} \rho_j(x_k)=0=\mathfrak{F}_j(\rho_j)(x_k)$, we only consider indices $j$ such that $|x_k-j/2N|\leq r$. 
Utilizing \eqref{def:extendFrho} at the beginning of this section, we get
\begin{multline} \label{2nddiff}
    |T_{\mathcal{K}_j} \overline{\rho}_j(x_k) - \overline{\mathfrak{F}(\rho_j)}(x_k)| \\
    \leq \frac{1}{\mathcal{W}_{x_k}}\sum_{l=1}^n \bigg|\int_{(l-1)/n}^{l/n} \big(\mathcal{G}^*(x_k,y)W(x_k,y)-\mathcal{G}^*(x_k,x_l)[{\bf A}_n]_{k,l}\big) \rho_j(x_l)\,\mathrm{d}y\bigg|.
\end{multline}
Observe that at most $\lceil 2rn\rceil$ indices $l$ participate in the sum above; namely those that satisfy
\begin{equation}\label{relevantind}
    I_l=\bigg[\frac{l-1}{n}, \frac{l}{n}\bigg) \cap \mathfrak{I}_{x_k} \not= \emptyset.
\end{equation}
Among all the indices that satisfy \eqref{relevantind}, there will be at most two ``boundary" $l$ such that
\begin{equation}\label{boundaryl}
    I_l = \bigg[\frac{l-1}{n}, \frac{l}{n}\bigg) \not\subset \mathfrak{I}_{x_k},
\end{equation}
on at most one of which,
\begin{equation}\label{leftboundary1}
    \mathcal{G}^*(x_k,x_l) = 0.
\end{equation}
Let $l^*$ denote this index and $l_*$ denote the other index that satisfies \eqref{boundaryl}. Let $S$ be the set of the remaining, ``interior", indices, i.e., those $l$ that satisfy \eqref{relevantind} but not \eqref{boundaryl}. Among these include $l=k$, in which case it follows directly from the definition that
\begin{equation}\label{leftboundary2}
    [{\bf A}_n]_{k,l} = [{\bf A}_n]_{k,k} = 0.
\end{equation}
We estimate an upper bound for \eqref{2nddiff} for each index $l$, using the results of Lemma~\ref{lem:Gsize}.
Consider first
\begin{equation*} 
    \begin{split}
        D_1 &:=\frac{1}{\mathcal{W}_{x_k}} \bigg|\int_{(l^*-1)/n}^{l^*/n} (\mathcal{G}^*(x_k,y)W(x_k,y)-\mathcal{G}^*(x_k,x_{l^*})[{\bf A}_n]_{k,l^*})\rho_j(x_{l^*})\,\mathrm{d}y\bigg|\\
        &= \frac{1}{\mathcal{W}_{x_k}} \bigg|\int_{(l^*-1)/n}^{l^*/n} \mathcal{G}^*(x_k,y)W(x_k,y)\rho_j(x_{l^*})\,\mathrm{d}y\bigg|,
    \end{split}
\end{equation*}
due to \eqref{leftboundary1}. We find obtain
\begin{equation}\label{D1bd}
    D_1 \leq \frac{1}{\mathcal{W}_{x_k}}\int_{(l^*-1)/n}^{l^*/n} |\mathcal{G}^*(x_k,y)W(x_k,y)||\rho_j(x_{l^*})|\,\mathrm{d}y \leq \frac{C\eta^{-1}N^2}{n}.
\end{equation}
Likewise, by virtue of \eqref{leftboundary2},
\begin{equation}\label{def:D2bd}
    \begin{split}
        D_2 &:=\frac{1}{\mathcal{W}_{x_k}} \bigg|\int_{(k-1)/n}^{k/n} (\mathcal{G}^*(x_k,y)W(x_k,y)-\mathcal{G}^*(x_k,x_k)[{\bf A}_n]_{k,k})\rho_j(x_k)\,\mathrm{d}y\bigg|\\
        &= \frac{1}{\mathcal{W}_{x_k}} \bigg|\int_{(k-1)/n}^{k/n} \mathcal{G}^*(x_k,y)W(x_k,y)\rho_j(x_k)\,\mathrm{d}y\bigg|\\
        &\leq \frac{C\eta^{-1}N^2}{n}.
    \end{split}
\end{equation}
Next, let
\begin{equation}\label{def:D3}
    D_3:=\frac{1}{\mathcal{W}_{x_k}}\sum_{l\not=k, l\in S} \bigg|\int_{(l-1)/n}^{l/n} (\mathcal{G}^*(x_k,y)W(x_k,y)-\mathcal{G}^*(x_k,x_l)[{\bf A}_n]_{k,l})\rho_j(x_l)\,\mathrm{d}y\bigg|.
\end{equation}
We estimate, on one of these intervals $I_l$, that
\begin{equation}\label{splitcal}
    \begin{split}
        &|\mathcal{G}^*(x_k,x_l)[{\bf A}_n]_{k,l}-\mathcal{G}^*(x_k,y)W(x_k,y)| \\
        &\leq |\mathcal{G}^*(x_k,x_l)-\mathcal{G}^*(x_k,y)| + |[{\bf A}_n]_{k,l}-W(x_k,y)|\, |\mathcal{G}^*(x_k,y)| \\
        &= |\mathcal{G}^*(x_k,x_l)-\mathcal{G}^*(x_k,y)| + |W(x_k, x_l)-W(x_k,y)|\, |\mathcal{G}^*(x_k,y)| \leq 
        \frac{CN^3  }{n} + \frac{CKN^2}{n},
    \end{split}
\end{equation}
by using Assumption~\ref{assump:regular}. Plugging this back in \eqref{def:D3}, we gain
\begin{equation}\label{D3finbd}
    D_3 \leq \eta^{-1}(1+K)\bigg(\sum_{l\not=k, l\in S}\frac{1}{n}\bigg(\frac{N^3}{n}\bigg)\bigg) \leq C\eta^{-1}(1+K)\,\frac{N^{3 -\alpha}}{n},
\end{equation}
since $\texttt{\#} S\leq\lceil 2rn \rceil$. Finally, let
\begin{equation*}
    D_4:=\frac{1}{\mathcal{W}_{x_k}}\bigg|\int_{(l_*-1)/n}^{l_*/n} (\mathcal{G}^*(x_k,y)W(x_k,y)-\mathcal{G}^*(x_k,x_{l_*})[{\bf A}_n]_{k,l_*})\rho_j(x_{l_*})\,\mathrm{d}y\bigg|.
\end{equation*}
Then following a similar calculation as in \eqref{splitcal}, we derive
\begin{equation*}
    |\mathcal{G}^*(x_k,x_{l_*})[{\bf A}_n]_{k,l_*}-\mathcal{G}^*(x_k,y)W(x_k,y)| \leq \frac{CN^3}{n} + \frac{CKN^2}{n}
\end{equation*}
and so
\begin{equation}\label{D4bd}
    D_4\leq C\eta^{-1}(1+K)\,\frac{N^3}{n^2}. 
\end{equation}
Taking into account \eqref{splitdiff}, \eqref{1stdiff}, \eqref{2nddiff}, \eqref{D1bd}, \eqref{def:D2bd}, \eqref{D3finbd}, \eqref{D4bd}, we arrive at \eqref{kerneldifferr}, as desired. \qed

\section{Proof of Theorem~\ref{thm:GNNran}} \label{sec:GNNran}

We maintain that \eqref{rkappa} holds. Furthermore, we will continue using the shorthand abbreviation defined in the previous section
\begin{equation*}
    \mathcal{G}^*(x,y)=\frac{\mathrm{d}^2}{\mathrm{d} y^2}(s_N\,\mathcal{G}_{r,\sigma})(x-y).
\end{equation*}
The proof will be a straightforward application of the bounded-difference inequality, see Theorem~\ref{thm:vershynin} below, and Theorem~\ref{thm:WNN}. As usual, we begin with some preliminary arrangements. 

First, recall that in the case of a simple random graph $G^{\mathrm{ran}}_n$ generated from a graphon $W$, the event of two vertices $v_k, v_l$ being connected by an edge is a random variable 
\begin{equation}\label{Ber}
    \xi_{k,l}\sim\mathrm{Bernoulli}(W(x_k, x_l)), \quad k > l.
\end{equation}
Since $W(x_k, x_l)=W(x_l,x_k)$, we have an undirected symmetry $\xi_{lk} = \xi_{k,l}$; 
nevertheless, each of the sets
\begin{equation*}
    \{\xi_{k,l}: k,l=1,\cdots,n \text{ s.t } k>l\} \quad\text{ and }\quad \{\xi_{k,l}: k,l=1,\cdots,n \text{ s.t } k<l\}
\end{equation*}
is a set of independent random variables. A realization of the random graph $G^{\mathrm{ran}}_n$ is then a simple graph whose associated adjacency matrix ${\bf A}_n$ satisfies
\begin{equation*}
    [{\bf A}_n]_{k,l}\in\{0,1\}.
\end{equation*}
Using this matrix, the graph filter $\mathfrak{F}_j$'s application on a graph signal $g_n$ has the value
\begin{equation*}
    \mathfrak{F}_j(g_n)(x_k) = \frac{\chi_{\{x_k - r \leq j/2N\}}}{n\mathcal{W}_{x_k}}
    \sum_{ l=1}^n\mathcal{G}^*(x_k,x_l)[{\bf A}_n]_{k,l}g_n(x_l).
\end{equation*}
Until $\xi_{k,l}$ are all realized, the quantity
\begin{equation}\label{ranfrakF}
    \mathfrak{F}_j^{\mathrm{ran}}(g_n)(x_k) := \frac{\chi_{\{x_k - r \leq j/2N\}}}{n\mathcal{W}_{x_k}}
    \sum_{l\not=k, l=1}^n\mathcal{G}^*(x_k,x_l)\xi_{k,l} g_n(x_l)
\end{equation}
is a random variable for each $k=1,\cdots,n$. 
Fix one such $k$. It suffices to consider $j=0,\cdots, 2N-1$ be such that $|x_k-j/2N|\leq r$. Then $\chi_{\{x_k - r \leq j/2N\}}=1$. Set $g_n = \rho_j$ in \eqref{ranfrakF}. 
We define the random variables
\begin{equation} \label{splitvar}
    \begin{split}
         Y_{j, \mathcal{L}_k} &:= \frac{1}{n\mathcal{W}_{x_k}}\sum_{l\in\mathcal{L}_k}\mathcal{G}^*(x_k,x_l)\xi_{k,l}\rho_j(x_l)\\
        Y_{j, \mathcal{U}_k} &:= \frac{1}{n\mathcal{W}_{x_k}}\sum_{l\in\mathcal{U}_k}\mathcal{G}^*(x_k,x_l)\xi_{k,l}\rho_j(x_l),
    \end{split}
\end{equation}
where 
\begin{align*} 
    \mathcal{L}_k &:=\{l<k: \mathcal{G}^*(x_k,x_l)>0, \rho_j(x_l)>0\}\\
    \mathcal{U}_k &:=\{l>k: \mathcal{G}^*(x_k,x_l)>0, \rho_j(x_l)>0\}.
\end{align*}
Note from \eqref{splitvar} that 
\begin{equation}\label{size}
    \texttt{\#}\mathcal{L}_k, \,\, \texttt{\#}\mathcal{U}_k \leq \lceil nr\rceil. 
\end{equation}
We first focus on $Y_{j, \mathcal{L}_k}$. Observe from \eqref{Ber} that
\begin{equation*} 
    \mathbb{E}[Y_{j, \mathcal{L}_k}] = \frac{1}{n\mathcal{W}_{x_k}}\sum_{l\in\mathcal{L}_k}\mathcal{G}^*(x_k,x_l)W(x_k,x_l)\rho_j(x_l).
\end{equation*}
Moreover, $Y_{j, \mathcal{L}_k}$ has the form of a function $g_{\mathcal{L}_k}$ of $L_k := \texttt{\#}\mathcal{L}_k$ one-dimensional variables in $[0,1]$, applied to $\{\xi_{k,l}: l\in\mathcal{L}_k\}$, where
\begin{equation*}
    g_{\mathcal{L}_k}((z_l)_{l\in\mathcal{L}_k}):= \frac{1}{n\mathcal{W}_{x_k}}\sum_{l\in\mathcal{L}_k}\mathcal{G}^*(x_k,x_l) \rho_j(x_l)z_l.
\end{equation*}
Such a function $g_{\mathcal{L}_k}$ is of bounded difference. Indeed, let $z, z'$ be two $L_k$-tuples that differ at only one coordinate entry. Then by \eqref{rkappa}, Assumption~\ref{assump:regular}, and Lemma~\ref{lem:Gsize},
\begin{equation}\label{gL}
    |g_{\mathcal{L}_k}(z)-g_{\mathcal{L}_k}(z')|\leq \frac{C\eta^{-1}N^{2}}{n}. 
\end{equation}
As mentioned earlier, the central element of our proof is the bounded-difference inequality, presented as follows. 

\begin{theorem} (adapted from \citep[Theorem 2.9.1]{vershynin2018high}) \label{thm:vershynin} Let $X_1, \cdots, X_L$ be independent random variables. Let $g:\mathbb{R}^L\to\mathbb{R}$. Assume that the value of $g(x)$ can be changed by at most $c_j>0$ under an arbitrary change of a single $j$th coordinate of $x\in\mathbb{R}^L$. Then for any $\delta>0$, we have
\begin{equation*}
    \mathbb{P}(|g(X)-\mathbb{E}[g(X)]|>\delta)\leq 2\exp\bigg(-\frac{2\delta^2}{\sum_{j=1}^L c_j^2}\bigg)
\end{equation*}
where $X=(X_1,\cdots,X_L)$.
\end{theorem}

\begin{proof}[Proof of Theorem~\ref{thm:GNNran}]
Let $\varepsilon\in (0,1)$. Applying Theorem~\ref{thm:vershynin} to $Y_{j, \mathcal{L}_k}$, using \eqref{size}, \eqref{gL}, we acquire
\begin{equation*}
    \mathbb{P}(|Y_{j, \mathcal{L}_k}-\mathbb{E}[Y_{j, \mathcal{L}_k}]|>\varepsilon)
    \leq 2\exp\bigg(-\frac{C\varepsilon^2\eta^2 n^2}{N^4 \lceil nr\rceil}\bigg)\leq 2\exp\bigg(-\frac{C\varepsilon^2\eta^2 n}{N^{4-\alpha}}\bigg).
\end{equation*}
A similar argument in the case of $\mathcal{U}_k$ would also result in
\begin{equation*}
    \mathbb{P}(|Y_{j, \mathcal{U}_k}-\mathbb{E}[Y_{j, \mathcal{U}_k}]|>\varepsilon)
    \leq 2\exp\bigg(-\frac{C\varepsilon^2\eta^2 n}{N^{4-\alpha}}\bigg).
\end{equation*}
Since $Y_{j, \mathcal{L}_k}+Y_{j, \mathcal{U}_k}=\mathfrak{F}^{\mathrm{ran}}_j(\rho_j)(x_k)$, we conclude that
\begin{equation}\label{basicReLU}
    \mathbb{P}(|\mathfrak{F}^{\mathrm{ran}}_j(\rho_j)(x_k)-\mathbb{E}[\mathfrak{F}^{\mathrm{ran}}_j(\rho_j)(x_k)]|>\varepsilon)\leq 4\exp\bigg(-\frac{C\varepsilon^2\eta^2 n}{N^{4-\alpha}}\bigg),
\end{equation}
for every $x_k\in \mathcal{X}_n$ and each $j = 0,\dots,2N - 1$ such that $|x_k - j/2N|\leq r$. 
Recall the normalization of $f$. 
Then summing over all such $j$ gives
\begin{align}
    \nonumber &\mathbb{P}\bigg(\bigg|\sum_{j=0}^{2N-1}f\Big(\frac{j}{2N}\Big)\mathfrak{F}^{\mathrm{ran}}_j(\rho_j)(x_k)-\sum_{j=0}^{2N-1}f\Big(\frac{j}{2N}\Big)\mathbb{E}[\mathfrak{F}^{\mathrm{ran}}_j(\rho_j)(x_k)]\bigg|>\varepsilon\bigg)\\
    \nonumber \leq &\mathbb{P}\bigg(\sum_{j=0}^{2N-1} |f\Big(\frac{j}{2N}\Big)||\mathfrak{F}^{\mathrm{ran}}_j(\rho_j)(x_k)-\mathbb{E}[\mathfrak{F}^{\mathrm{ran}}_j(\rho_j)(x_k)]|>\varepsilon\bigg)\\
    \nonumber \leq &\mathbb{P}\bigg(\bigg(\sum_{j=0}^{2N-1} |\mathfrak{F}^{\mathrm{ran}}_j(\rho_j)(x_k)-\mathbb{E}[\mathfrak{F}^{\mathrm{ran}}_j(\rho_j)(x_k)]|^2\bigg)^{1/2}>\varepsilon\bigg)\\
    \nonumber \leq &\sum_{j=0}^{2N-1} \mathbb{P}\bigg(|\mathfrak{F}^{\mathrm{ran}}_j(\rho_j)(x_k)-\mathbb{E}[\mathfrak{F}^{\mathrm{ran}}_j(\rho_j)(x_k)]|>\varepsilon (rN)^{-1/2}\bigg)\\
    \label{chainofprobs} \leq & CN^{1-\alpha} \exp\bigg(-\frac{C\varepsilon^2\eta^2 n}{N^{5-2\alpha}}\bigg).
\end{align}
The second inequality above follows from the Cauchy-Schwarz inequality, the third from the $L^p$-embedding for finite measures, and the last from \eqref{basicReLU}. Note that, by design
\begin{equation*}
    \sum_{j=0}^{2N-1}f\Big(\frac{j}{2N}\Big)\mathfrak{F}^{\mathrm{ran}}_j(\rho_j)(x_k)
\end{equation*}
is the GNN network output $\Psi_{n,f}$ on the random graph $G^{\mathrm{ran}}_n$ (see \eqref{GNNoutput}). In this case
\begin{equation*}
    \overline{\Psi}_{n,f}(x) = \sum_{k=1}^n \bigg(\sum_{j=0}^{2N-1}f\Big(\frac{j}{2N}\Big)\mathfrak{F}^{\mathrm{ran}}_j(\rho_j)(x_k)\bigg)\chi_{I_k}(x),
\end{equation*}
and so from \eqref{def:extendFrho},
\begin{equation*}
    \begin{split}
        \mathbb{E}[\overline{\Psi}_{n,f}(x)] &= 
         \sum_{j=0}^{2N-1} f\Big(\frac{j}{2N}\Big) \bigg(\sum_{k=1}^n \mathbb{E}[\mathfrak{F}^{\mathrm{ran}}_j(\rho_j)(x_k)]\chi_{I_k}(x)\bigg) \\
        &= \sum_{j=0}^{2N-1} \frac{f\Big(\frac{j}{2N}\Big)}{n\mathcal{W}_{x_k}}\sum_{k=1}^{n} \bigg(\sum_{l\not=k,l=1}^n \mathcal{G}^*(x_k,x_l)W(x_k,x_l)\rho_j(x_l)\bigg)\chi_{I_k}(x) \\
        &= \sum_{j=0}^{2N-1}f\Big(\frac{j}{2N}\Big)\overline{\mathfrak{F}_j(\rho_j)}(x).
    \end{split} 
\end{equation*}
Therefore, since every $x\in [0,1]$ belongs to some interval $I_k = [x_k, x_{k+1})$, we deduce from \eqref{chainofprobs} that
\begin{equation}\label{probconsequence}
    \sup_{x\in [0,1]} \bigg|\overline{\Psi}_{n,f}(x) - \sum_{j=0}^{2N-1}f\Big(\frac{j}{2N}\Big)\overline{\mathfrak{F}(\rho_j)}(x) \bigg|
    \leq \varepsilon
\end{equation}
with probability at least $1-CnN^{1-\alpha}\exp(-\frac{C\varepsilon^2\eta^2 n}{N^{5-2\alpha}})$. 
Moreover, from the proof of Theorem~\ref{thm:GNNdet}, 
\begin{equation}\label{WNNrecall}
    \bigg\| \sum_{j=0}^{2N-1}f\Big(\frac{j}{2N}\Big)\overline{\mathfrak{F}_j(\rho_j)} - f \bigg\|_{L^1([r,1-r];\mathbb{C}^m)} 
    \leq C\eta^{-2}(1+K)\,\varepsilon 
\end{equation}
holds whenever $n\geq \lceil\varepsilon^{-10/3}\rceil$. Then by combining \eqref{probconsequence}, \eqref{WNNrecall}, and recalling that $N=\lceil\varepsilon^{-10/9}\rceil$, we arrive at
\begin{equation*}
    \| \overline{\Psi}_{n,f} - f \|_{L^1([r,1-r];\mathbb{C}^m)} 
    \leq C\eta^{-2}(1+K)\,\varepsilon
\end{equation*}
with a probability at least
\begin{equation*}
    1 - 2n\varepsilon^{10(1-\alpha)/9}\exp\Big(-C\eta^2 n \varepsilon^{10(5-2\alpha)/9 + 2}\Big),
\end{equation*}
where $\alpha=0.96$, 
so as long as $n\geq \varepsilon^{-10/3}$. We conclude the proof.
\end{proof} 

\section{Discussion}\label{sec:Discussion}

We present novel results on WNN generalizability and GNN transferability. First, we design our network architectures by leveraging the sampling routine from Theorem~\ref{thm:TsamplingregwG}. For our WNNs, we establish sample complexities to accurately approximate graphon signals. Following this, we demonstrate GNN transferability across graphs from the same graphon family, including deterministic and random graphs.

These results mark an initial step towards a unified theory of graphons and their application to GNNs. Future research aims to broaden the scope beyond graphons that satisfy Assumption~\ref{assump:regular}. For instance, popular graphs like stochastic block models \citep{airoldi2013stochastic, sischka2022stochastic} do not fit our assumption, but could be approached by partitioning their domains into compact blocks. Generalizing our findings to include arbitrary compact sets $\Omega$ in place of $[0,1]^2$ is another promising direction, as discussed in \citep{Janson}. This approach would enhance the applicability of our results across diverse graphon structures.

Given that our work specifies a concrete network architecture, we have outlined a straightforward GNN/WNN design for practical experimentation. Our results are presented in the worst-case scenario, so it is essential to benchmark our network's performance on real-world data. It would be valuable to compare our architecture with other WNN and GNN models from the literature. These experiments could offer insights into improving architectures and refining bounds beyond what is reported here.

Finally, it is worth noting that graphons inherently describe {\it dense} graphs. Effectively expanding our results to cover a wider range of graph structures would require leveraging extended graphon theories to capture sparse graph characteristics \citep{klopp2017oracle,borgs2017sparse,lunde2023subsampling,fabian2023learning}. 
This poses significant challenges and remains an area for future research.

\section*{Acknowledgements}

AMN was supported by the Austrian Science Fund (FWF) Project P-37010, and JJB was supported by an NSERC Discovery grant.


\bibliographystyle{plain}

\appendix 

\section{Proofs of Proposition~\ref{prop:gpsampling} and Lemma~\ref{lem:tech}} \label{sec:Tsamp}

Let us begin by identifying $\mathbb{T}\cong [0,1)$, in the sense that points in $\mathbb{T}$ are identified as orbits of points in $[0,1)$ under a shift by $1$ unit:
\begin{equation*}
    \mathbb{T}\ni [x] :=\{x+k: k\in\mathbb{Z}\} \quad\text{ where }\quad x\in [0,1).
\end{equation*}
To reduce the amount of notation, we will simply write $x\in\mathbb{T}$, with an understanding that either $x$ is an abstract point in $\mathbb{T}$ or $x$ is the unique real-valued representative of $[x]$ in $[0,1)$. 
A continuous function $\phi$ on $\mathbb{T}$, therefore, corresponds one-to-one to a continuous function $g$ on $[0,1]$ such that $g(0)=g(1)$. We associate with $\mathbb{T}$ the following modulo addition, $+: \mathbb{T}\times\mathbb{T}\to\mathbb{T}$, under which
\begin{equation}\label{def:gpadd}
    (x,y) \overset{+}{\mapsto} x+y \mod 1,
\end{equation}
where $x,y\in [0,1)$, and, by an abuse of notation, the second $+$ denotes the usual addition on $\mathbb{R}$. It is known that $(\mathbb{T},+)$ is a compact Abelian group. 

\begin{definition} \label{def:Pontryagindual}
A character $\phi$ of $\mathbb{T}$ is a group homomorphism from $\mathbb{T}$ to the multiplicative group $(\mathbb{C}^*,\times)$, where $\mathbb{C}^*:=\mathbb{C}\setminus\{0\}$. The Pontryagin dual $\widehat{\mathbb{T}}$ of $\mathbb{T}$ is the set of all these characters. 
\end{definition}

It is known that $\widehat{\mathbb{T}}\cong\mathbb{Z}$, as additive groups, and that a character of $\mathbb{T}$ acts on $\mathbb{T}$ as follows
\begin{equation*}
    \phi: x\mapsto e^{i2\pi kx}, \quad\text{ for some }\quad k\in\mathbb{Z}.
\end{equation*}
Evidently, $|\phi(x)|=1$ for all $\phi\in\widehat{\mathbb{T}}$, $x\in\mathbb{T}$. The concept of group characters allows one to define the duality between $\mathbb{T}$ and $\mathbb{Z}$ to be
\begin{equation}\label{dual}
    \langle x,k\rangle := e^{i2\pi kx}.
\end{equation}
It follows from Fourier theory that if $f\in L^2(\mathbb{T})$, then 
\begin{equation*}
    \hat{f}(k) = \int_0^1 f(x)e^{-i2\pi kx}\,\mathrm{d}x = \int_0^1 f(x)\langle x,-k\rangle\,\mathrm{d}x 
\end{equation*}
for every $k\in\mathbb{Z}$, and that,
\begin{equation}\label{TFinv}
    f(x) = \sum_{k\in\mathbb{Z}} \hat{f}(k)e^{i2\pi kx}
\end{equation}
for almost every $x \in \mathbb{T}$.

If $H$ is a nontrivial discrete subgroup of $\mathbb{T}$, then $H$ must take the form 
\begin{equation}\label{groupH}
    H=\{x\in [0,1): Nx = 0 \mod 1 \}=\{0, 1/N,\cdots, (N-1)/N\}
\end{equation}
for some integer $N \geq 1$. Equivalently, $H$ can be identified as the group of $N$th roots of unity, $H=\{z\in\mathbb{C}: z^{N}=1\}$. Let $\Lambda$ be a discrete subgroup lattice of $\mathbb{Z}=\widehat{\mathbb{T}}$ that is the {\it annihilator} of $H$, defined by 
\begin{equation*}
    \Lambda=H^{\perp}:=\{k\in\mathbb{Z}: \langle x,k\rangle = e^{i2\pi kx} = 1, \forall\, x\in H\}.
\end{equation*}
Suppose $\Lambda=M\mathbb{Z}$ for some positive integer $M$. Then by \eqref{dual} and \eqref{groupH}
\begin{equation*}
    e^{\frac{i2\pi Mk}{N}} = 1 = \cdots = e^{\frac{i2\pi(N-1)Mk}{N}}, \quad \forall\, k\in\mathbb{Z}
\end{equation*}
which implies $M=N$ and $\Lambda=H^{\perp}=N\mathbb{Z}$.\\

By equipping $\mathbb{T}$ with the Haar measure $\lambda$ that is the usual Lebesgue measure on $[0,1)$, so $\lambda(\mathbb{T})=1$, we can identify $L^2(\mathbb{T})=L^2([0,1))= L^2([0,1])$. 
We present the following Kluvan\'ek sampling theorem, a key ingredient of our proofs. 

\begin{theorem}\citep[Theorem~1.5]{benedetto2001modern} \label{thm:Kluvanek}
Let $H\subset G$ be a discrete subgroup of a locally compact Abelian (LCA) group $G$, with a discrete annihilator subgroup $H^{\perp}\subset\hat{G}$. Let $E\subset\hat{G}$ be any subset of finite Haar measure for which the canonical surjective map
\begin{equation*}
    h: \hat{G} \to \hat{G}/H^{\perp}
\end{equation*}
restricted to $E$ is a bijection, and define the sampling function
\begin{equation*} 
    s_{E}(x):=\int_{E} \langle x,\gamma\rangle\,\mathrm{d}\hat{\lambda}(\gamma), \quad\forall\,x\in G,
\end{equation*}
where $\hat{\lambda}$ is a normalized Haar measure  on $\hat{G}$ such that $\hat{\lambda}(E)=1$. Let $f\in L^2(G)$ and assume $\hat{f}=0$ almost everywhere off $E$. Then the following hold.
\begin{enumerate}
    \item There exists a continuous function $f_{c}$ on $G$ such that $f=f_{c}$ almost everywhere.
    \item If $f$ is continuous on $G$, then
    \begin{equation}\label{Kluvaneksamp}
        f(x) = \sum_{y\in H} f(y)s_{E}(x-y)
    \end{equation}
    where the convergence of the sums is in $L^2$-norm and uniformly on $G$. Furthermore, the Gaussian quadrature formula
    \begin{equation}\label{Gquad}
        \|f\|^2_{L^2(G)} = \sum_{y\in H}|f(y)|^2
    \end{equation}
    is valid.
\end{enumerate}
\end{theorem}

With Theorem~\ref{thm:Kluvanek} at our disposal, we provide a proof of Proposition~\ref{prop:gpsampling}.

\begin{proof}[Proof of Proposition~\ref{prop:gpsampling}]
Without loss of generality, we assume $m=1$. 
Let $f \in \mathcal{B}_\mathfrak{m}$, and fix an $N \geq \mathfrak{m}$. We specify Theorem~\ref{thm:Kluvanek} to the case of $G=\mathbb{T}$, 
\begin{equation*}
    H = \{0, 1/2N, \cdots, (2N-1)/2N\} \subset\mathbb{T},
\end{equation*}
and $f$ continuous on $\mathbb{T}$ (see \eqref{contf}) such that $\hat{f}=0$ off $E=H^{\perp}=B_N$. Let $\hat{\lambda}$ be the scaled counting measure on finite subsets $A \subset \mathbb{Z}$ given by 
\begin{equation*}
    \hat{\lambda}(A) := \frac{\texttt{\#} A}{2N}.
\end{equation*}

Clearly, $\hat{\lambda}(B_N)=1$, and $s_{E}=s_N$ in \eqref{def:sampf}, so that
\begin{equation}\label{sfunc}
    s_N(x) = \frac{1}{2N}\sum_{k=-N}^{N-1} e^{i2\pi kx}.
\end{equation}
Then by \eqref{Kluvaneksamp}, for every $x\in\mathbb{T}$, 
\begin{equation}\label{sampseries}
    f(x) = \sum_{j=0}^{2N-1} f\Big(\frac{j}{2N}\Big)s_N(x-j/2N) \\
    = \frac{1}{2N}\sum_{j=0}^{2N-1} f\Big(\frac{j}{2N}\Big) \bigg(\sum_{k=-N}^{N-1} e^{i2\pi (x-j/2N)k}\bigg),
\end{equation}
as to be shown. This completes the proof of the proposition.
\end{proof} 

More can be inferred from \eqref{Gquad} and \eqref{sampseries}, which leads us to the statement of Lemma~\ref{lem:tech} in Subsection~\ref{sec:Fsamptheory}. However, we first need to revisit some concepts in Fourier theory on finite groups. Let $G$ be a finite abelian group and set $M = \texttt{\#} G$. Then, $\hat{G}\cong G\cong\mathbb{Z}_M$ \citep{dummit2004abstract}, and so to simplify our presentation we will treat both $G$ and $\hat{G}$ as $\mathbb{Z}_M$. We define the Fourier transform on $L^2(G)$ by
\begin{equation*} 
    \mathcal{F}_{G}(f)(v) := \frac{1}{M}\sum_{u\in G} f(u)e^{-i2\pi uv}, \quad \forall v\in \hat{G}=\mathbb{Z}_M,
\end{equation*}
where $f:\mathbb{Z}_M\to\mathbb{C}$, and $uv$ denotes usual scalar multiplication. The inverse Fourier transform $\mathcal{F}^{-1}_{G}$ on $L^2(\hat{G})$ is then
\begin{equation*} 
   \mathcal{F}^{-1}_{G}(f)(u) := \sum_{v\in\hat{G}} f(v)e^{i2\pi uv}, \quad \forall u\in G=\mathbb{Z}_M,
\end{equation*}
where $f: \hat{G}=\mathbb{Z}_M\to\mathbb{C}$. Together these transformations satisfy 
\begin{equation}\label{id}
    \mathcal{F}^{-1}_{G}\circ\mathcal{F}_{G} = {\rm Id}_{L^2(G)} \quad\text{ and }\quad \mathcal{F}_{G}\circ\mathcal{F}^{-1}_{G} = {\rm Id}_{L^2(\hat{G})}.
\end{equation}
We now provide the proof of Lemma~\ref{lem:tech}.

\begin{proof}[Proof of Lemma~\ref{lem:tech}]
We again assume $m=1$. First, \eqref{GquadforT} is a direct consequence of \eqref{Gquad} and Plancherel's theorem. Hence, we need only prove \eqref{claim_finv}. Setting $\hat{f}(k)=0$ if $k\in\mathbb{Z}\setminus B_N$ in \eqref{TFinv}
, we obtain
\begin{equation}\label{FTfx}
    f(x) = \sum_{k=-N}^{N-1} \hat{f}(k)e^{i2\pi kx}, \quad \forall x\in\mathbb{T}.
\end{equation}
Letting $H=\{0,1/2N,\cdots, (2N-1)/2N\}\subset\mathbb{T}\cong [0,1)$ inherit the group addition \eqref{def:gpadd}, it follows that $H\cong\mathbb{Z}_{2N}$ and so $\hat{H}\cong H\cong\mathbb{Z}_{2N}$. We let
\begin{equation*}
    \hat{H} = B_N = \{-N, \cdots, 0, \cdots, N-1\},
\end{equation*}
with the group addition $k+l \equiv k+l \mod 2N$. 

Define $\tilde{g}: \hat{H}\to\mathbb{C}$ to be such that for each $k\in B_N$ we have $\tilde{g}(k) := \hat{f}(k)$. Then, from \eqref{id} we have $\tilde{g}=\mathcal{F}_{H}(g^*)$ for some $g^*: H\to\mathbb{C}$, where
\begin{equation}\label{FTgp1}
    \mathcal{F}_{H}(g^*)(k) = \frac{1}{2N}\sum_{x\in H}g^*(x)e^{-i2\pi kx} = \tilde{g}(k) = \hat{f}(k),
\end{equation}
and so
\begin{equation}\label{FTgp2}
    g^*(x) = \sum_{k=-N}^{N-1}\tilde{g}(k)e^{i2\pi kx} = \sum_{k=-N}^{N-1}\hat{f}(k)e^{i2\pi kx}, \quad \forall x\in H. 
\end{equation}
Comparing \eqref{FTfx} with \eqref{FTgp2} we conclude that $g^*(j/2N) = f\Big(\frac{j}{2N}\Big)$ for $j=0,\cdots, 2N-1$. Therefore, from \eqref{FTgp1},
\begin{equation*} 
    \hat{f}(k) = \frac{1}{2N}\sum_{j=0}^{2N-1}f\Big(\frac{j}{2N}\Big)e^{-i2\pi kj/2N}, \quad\forall k\in B_N.
\end{equation*}
Since it is clear from the definition that $\hat{f}(k)=0$ for $k\in B_N\setminus B_\mathfrak{m}$, we complete the proof.    
\end{proof} 

\end{document}